\documentclass{article}

\usepackage[final, nonatbib]{neurips_2021}

\usepackage[utf8]{inputenc} %
\usepackage[T1]{fontenc}    %
\usepackage{hyperref}       %
\usepackage{enumitem}
\usepackage{url}            %
\usepackage{booktabs}       %
\usepackage{amsfonts, mathtools, amsmath,amssymb, amsthm}
\usepackage{nicefrac}       %
\usepackage{microtype}      %
\usepackage{xcolor}     %
\definecolor{col1}{HTML}{d7191c}
\definecolor{col2}{HTML}{fdae61}
\definecolor{col3}{HTML}{abd9e9}
\definecolor{col4}{HTML}{2c7bb6}
\usepackage[giveninits=true,url=false,doi=false,isbn=false,eprint=true,datamodel=mrnumber,sorting=nty,maxcitenames=4,maxbibnames=99,backref=false,block=space,backend=biber,bibstyle=phys, biblabel=brackets]{biblatex} 
\AtEveryBibitem{\clearfield{month}}
\AtEveryCitekey{\clearfield{month}} 
\renewbibmacro{in:}{}
\DeclareFieldFormat[article]{title}{\emph{#1}} 
\DeclareFieldFormat{mrnumber}{\ifhyperref{\href{http://www.ams.org/mathscinet-getitem?mr=#1}{\nolinkurl{MR#1}}}{\nolinkurl{#1}}}
\DeclareFieldFormat{pmid}{\ifhyperref{\href{https://www.ncbi.nlm.nih.gov/pubmed/#1}{\nolinkurl{PMID#1}}}{\nolinkurl{#1}}}
\DeclareFieldFormat{eprint}{\ifhyperref{\href{https://arxiv.org/abs/#1}{\nolinkurl{arXiv:#1}}}{\nolinkurl{#1}}}
\renewbibmacro*{doi+eprint+url}{%
  \iftoggle{bbx:doi}{\printfield{doi}}{}
  \newunit\newblock%
  \printfield{mrnumber}%
  \newunit\newblock%
  \printfield{pmid}%
  \newunit\newblock%
  \printfield{eprint}%
  \iftoggle{bbx:url}{\usebibmacro{url+urldate}}{}}
\bibliography{refs}
\usepackage{bm}
\usepackage[labelfont=rm]{subfig}
\usepackage{pgfplots,tikz}
\pgfplotsset{compat=1.17}
\pgfplotsset{
/pgfplots/ybar legend/.style={
/pgfplots/legend image code/.code={%
\draw[draw=none,##1,/tikz/.cd,bar width=0.1cm,yshift=-0.2em,bar shift=0.5*\pgfplotbarwidth]
plot coordinates {(0.5*\pgfplotbarwidth,0.6em) (2.5*\pgfplotbarwidth,0.4em) (4.5*\pgfplotbarwidth,0.2em)};},
}
}
\usepackage{tikz} 
\usetikzlibrary{arrows,decorations.markings}
\DeclareMathOperator{\Tr}{Tr}
\DeclareMathOperator{\rank}{rank}
\DeclareMathOperator{\Var}{Var}
\newcommand{\dif}{\operatorname{d}\!{}}
\newtheorem{thm}{Theorem}[section]

\newtheorem{lem}[thm]{Lemma}
\newtheorem{prop}[thm]{Proposition}

\newtheorem{rmk}[thm]{Remark}
\DeclarePairedDelimiter{\braket}{\langle}{\rangle}%
\DeclarePairedDelimiter{\abs}{\lvert}{\rvert}%
\DeclarePairedDelimiter{\norm}{\lVert}{\rVert}%

\title{Analysis of one-hidden-layer Neural Networks via the Resolvent Method}

\author{
  Vanessa Piccolo \\
  ETH Zurich (current affiliation: ENS Lyon)\\
  \texttt{vanessa.piccolo@ens-lyon.fr} \\
  \And
  Dominik Schröder\\
  Institute for Theoretical Studies\\ 
  ETH Zurich\\
   \texttt{dschroeder@ethz.ch} 
}

\begin{document}

\maketitle

\begin{abstract}
In this work, we investigate the asymptotic spectral density of the random feature matrix \(M = Y Y^\ast\) with \(Y = f(WX)\) generated by a single-hidden-layer neural network, where \(W\) and \(X\) are random rectangular matrices with i.i.d.\ centred entries and \(f\) is a non-linear smooth function which is applied entry-wise. We prove that the Stieltjes transform of the limiting spectral distribution approximately satisfies a quartic self-consistent equation, which is exactly the equation obtained by Pennington and Worah~\cite{PW17} and Benigni and Péché~\cite{1904.03090} with the moment method. We extend the previous results to the case of additive bias \(Y=f(WX+B)\) with \(B\) being an independent rank-one Gaussian random matrix, closer modelling the neural network infrastructures encountered in practice. Our key finding is that in the case of additive bias it is impossible to choose an activation function preserving the layer-to-layer singular value distribution, in sharp contrast to the bias-free case where a simple integral constraint is sufficient to achieve isospectrality. To obtain the asymptotics for the empirical spectral density we follow the \emph{resolvent method} from random matrix theory via the cumulant expansion. We find that this approach is more robust and less combinatorial than the moment method and expect that it will apply also for models where the combinatorics of the former become intractable. The resolvent method has been widely employed, but compared to previous works, it is applied here to non-linear random matrices.
\end{abstract}
\section{Introduction}
Machine learning has seen many successful achievements in recent years. Applications in face identification, object and speech recognition, translation, email spam filtering, navigation, medical diagnosis, etc.\ have proved the enormous potential of machine learning for day-to-day live~\cite{26017442, MR3617773}. Deep neural networks have turned out to be a particularly powerful machine learning method, and understanding the theoretical underpinning of their success has received tremendous attention in mathematics, physics and computer science.

A fully-connected, feed-forward neural network with \(L\) hidden layers of dimensions \(n_1, \dots, n_L\) can be modelled as follows:
\[f_\theta(\bm{x}) = \beta^\ast f (W^{(L)} \, f(W^{(L-1)} \,f( \dots f(W^{(1)} \bm{x}) \dots ))) \in \mathbb{R}^{d},\]
where \(\bm{x} \in \mathbb{R}^{n_0}\) denotes the input data vector and \(f\colon \mathbb{R} \to \mathbb{R}\) is a non-linear activation function which is applied entry-wise. We denote the parameters of the network by \(\theta \coloneqq (W^{(1)}, \dots, W^{(L)}, \beta)\), where \(W^{(l)} \in \mathbb{R}^{n_l \times n_{l-1}}\) for \(1 \leq l \leq L\) and \(\beta \in \mathbb{R}^{n_L \times d}\) are the matrices of the weights. In the classical setting of supervised learning, we are given a training set of (say) \(m\) samples of input feature vectors \(\bm{x}_i \in \mathbb{R}^{n_0}\) with associated target vectors \(\bm{z}_i \in \mathbb{R}^d\). For example, \(\bm{x}_i\) may encode the pixels of a photograph of an animal and the target \(\bm{z}_i\) may label the species of the animal in the image. Roughly speaking, the goal of supervised learning is to learn the mapping between the feature and the target vectors based on a given training set in order to predict the output of new unlabelled data. Let \(X = (\bm{x}_1 \, \dots \, \bm{x}_m) \in \mathbb{R}^{n_0 \times m}\) be the matrix of the data and let \(Z = (\bm{z}_1 \, \dots \, \bm{z}_m) \in \mathbb{R}^{d \times m}\) be the target matrix. Then, the aim of the network is to find optimal parameters \(\theta\) such that \(f_\theta(X)= (f_\theta (\bm{x}_1), \dots, f_\theta (\bm{x}_m)) \in \mathbb{R}^{d \times m}\) approximates the target \(Z\) optimally. During the training phase, weights are adjusted in order to minimize the empirical risk \(\mathcal{R}(\theta) = \mathbf{E} \, \mathcal{L}(f_\theta(X), Z)\), where \(\mathcal{L}(\cdot, \cdot)\) is a given loss function, usually involving some penalty for large weights \(\theta\) in order to avoid \emph{over-fitting}. Stochastic gradient descent (SGD) and its variants with \emph{back-propagation} are the most commonly used algorithms for training multilayer networks by iteratively updating the parameters into the direction of the negative of the gradient of the empirical risk. For a much more complete survey, we refer the reader to~\cite{MR3617773}.

In the present paper, we will focus on a single-hidden-layer neural network of the form \(f_\theta(X) = \beta^\ast Y\) with \(Y = f(WX)\). This model was first studied by Louart, Liao, and Couillet~\cite{MR3784498} for the case where the data matrix \(X\) is deterministic and \(W\) is a matrix of random weights (in particular, the weights are given by functions of standard Gaussian random variables), and by Pennington and Worah~\cite{PW17} for the case where \(X\) and \(W\) are independent random matrices with both centred Gaussian entries. In both papers, the matrix \(\beta \in \mathbb{R}^{n_1 \times d}\) is the only parameter to be learned and is chosen as the unique minimizer of the ridge-regularized least squares loss function
\[ \mathcal{L}(f_\theta(X), Z) = \frac{1}{2dm} \norm{Z - \beta^\ast Y}^2_F + \gamma \norm{\beta}_F^2,\]
where \(\gamma>0\) is the \emph{learning rate}. The unique minimizing weight matrix \(\hat{\beta}\) is then equal to
\(\hat{\beta} = Y G(-\gamma) Z^\ast\), where \[G(-\gamma) = \left ( \frac{1}{m} Y^\ast Y + \gamma \right )^{-1}\] is the resolvent of \(\frac{1}{m}Y^\ast Y\). As proved in~\cite{MR3784498, PW17}, the expected training loss \(E_{\text{train}}\) is related to \(-\gamma \frac{\partial}{\partial \gamma} G(-\gamma)\), and thus also to the Stieltjes transform of the limiting spectral measure of \(\frac{1}{m}Y^\ast Y\). Here the Stieltjes transform \(m_\mu\) of a probability measure \(\mu\) on \(\mathbb{R}\) is defined as \(m_\mu(z) \coloneqq \int_\mathbb{R} (x-z)^{-1}\dif\mu(x)\) for \(z\in\mathbb{C} \) such that \(\Im z \geq 0\), and for \(\mu_{n_1}\) being the empirical probability measure of the \(n_1\) eigenvalues of \(\frac{1}{m}Y^\ast Y\) is related to the resolvent via \(m_{\mu_{n_1}}(z)=\frac{1}{n_1}\Tr G(z)\). The performance of one-hidden-layer neural networks depends on the asymptotic spectral properties of the matrix \(\frac{1}{m}Y^\ast Y\). Pennington and Worah~\cite{PW17} investigated the limiting spectral measure of the random matrix \(M = \frac{1}{m} Y Y^\ast\) and derived the quartic self-consistent equation
\begin{equation} \label{quartic eq}
  \begin{split}
    1 + zg_\infty  &= \theta_1(f) g_\infty \left ( 1- \frac{\phi}{\psi} (1 + zg_\infty) \right ) - \frac{\theta_2(f)}{\psi} g_\infty (1+zg_\infty) \left ( 1- \frac{\phi}{\psi} (1 + zg_\infty) \right ) \\
  & \quad + \frac{\theta_2(f)  (\theta_1(f) - \theta_2(f) ) }{\psi}g_\infty^2 \left ( 1- \frac{\phi}{\psi} (1 + zg_\infty) \right )^2,
  \end{split}
\end{equation}
where \(g_\infty(z) \coloneqq \lim_{n_1 \to \infty} g(z)\) and \(g(z)\coloneqq\frac{1}{n_1}\Tr G(z)\) is the Stieltjes transform, which is approximately satisfied, \(g_\infty(z)\approx g(z)\), by \(g(z)\) in case of Gaussian \(W,X\). It is notable that the asymptotic spectrum of \(Y^\ast Y\) for large dimensions such that \(n_0/m\to\phi\in(0,\infty)\) and \(n_0/n_1\to\psi\in(0,\infty)\) depends on the non-linear function \(f\) only through two integral parameters \(\theta_1(f)\) and \(\theta_2(f)\), where \(\theta_1(f)\) is the Gaussian mean of \(f^2\) and \(\theta_2(f)\) is the square of the Gaussian mean of \(f'\) (c.f.~\eqref{theta def}). Benigni and P\'{e}ch\'{e}~\cite{1904.03090} then extended this model to random matrices \(W\) and \(X\) with general i.i.d.\ centred entries, and obtained the same self-consistent equation~\eqref{quartic eq}. We mention that~\eqref{quartic eq} may be reduced, for some special cases, to the quadratic equation that is satisfied by the Stieltjes transform \(m_{\mu_{MP}}\) of the Marchenko-Pastur distribution \(\mu_{MP}\)~\cite{MR208649}. This means that for some activation functions, the non-linear random matrix model has the same limiting spectral distribution as that of sample covariance matrices \(XX^\ast\) (in other cases, the equation can simplify to the cubic equation approximately satisfied by product Wishart matrices~\cite{1401.7802, MR3251989}). This can be generalised to multilayer networks:~\cite{PW17} found experimentally that the singular value distribution is preserved through multiple layers by activation functions with \(\theta_2(f)=0\) and is given by the Marchenko-Pastur distribution in each layer. This conjecture was proved in~\cite{1904.03090} for the general case of bounded activation functions. Moreover,~\cite{MR3784498} performed a spectral analysis on the Gram matrix model with general training data and proved that, in the large dimensional regime, the resolvent of \(Y^\ast Y\) has a similar behaviour as that observed in sample covariance matrix models. This was extended in~\cite{LC18} by considering Gaussian mixture of data. We also refer to the recent paper~\cite{1805.08295}. In the context of multilayer feedforward neural networks, Fan and Wang~\cite{FW20} analysed the eigenvalue distribution of the Gram matrix model, where the weights are at random and the input vectors are assumed to be approximately pairwise orthogonal. In particular, they showed that the limiting spectral distribution converges to a deterministic limit and, at each intermediate layer, this limit corresponds to the Marchenko-Pastur map of a linear transformation of that of the previous layer. 

In recent years, there has been some progress in the asymptotic analysis of the eigenvalue distribution of another Gram matrix, the so-called Neural Tangent Kernel (NTK). Consider a multilayer neural network and denote by \(J = \nabla_\theta f_\theta(X)\) the Jacobian matrix of the network outputs with respect to the weights \(\theta\). Then, the NTK is the Gram matrix of \(J\), defined by \(K^\text{NTK}=J^\ast J\). It was shown in~\cite{JGH18} that the NTK at random initialization converges, in the infinite-width limit, to a deterministic kernel and it remains constant during the whole training time of the network. Subsequently,~\cite{PW18} analysed the spectrum of the sample covariance matrix \(J J^\ast\) in a single-hidden-layer neural network, and provided an exact asymptotic characterization of the spectral distribution of \(J J^\ast\) with random Gaussian weights and data. Recently,~\cite{FW20} proved that the limiting spectral measure of the NTK converges to a deterministic measure, which may be described by recursive fixed-point equations that extend the Marchenko-Pastur distribution.\\

The present paper is structured as follows. In the first part we consider the non-linear random matrix model studied in~\cite{1904.03090} and we compute its asymptotic spectral density. We follow the resolvent method via the cumulant expansion which, together with the moment method, is a standard approach to obtain the asymptotics for the empirical spectral density. In particular, we compute the self-consistent equation that is approximately satisfied by the Stieltjes transform of the limiting spectral distribution. This is a quartic equation and is the same as that found in~\cite{1904.03090}. In~\cite{1904.03090, PW17} the authors relied on the method of moments: they approximated general non-linear functions by polynomial ones and then computed the asymptotics of high moments \(\mathbf{E} \Tr(Y^k)\) with \(Y = f(WX)\) to obtain the limiting measure via its moments. Conversely, we approach matters in a more robust and less combinatorial fashion by applying the resolvent method: we consider \(Y\) as a random matrix with correlated entries and then we directly derive a self-consistent equation for its resolvent. In particular, we prove that the random matrix \(Y\) has cycle correlations, in the sense that the joint cumulant does not vanish when the random variables \(Y_{ij}\)'s are joined by a cycle graph. We find that the variance of \(Y_{ij}\) is given by the parameter \(\theta_1(f)\), whereas for \(k>1\) the cumulants \(\kappa(Y_{i_1 i_2}, Y^\ast_{i_2 i_3},Y_{i_3i_4}, \ldots, Y^\ast_{i_{2k} i_1})\) are powers of \(\theta_2(f)\). We note that in the random matrix literature matrices with general decaying correlations have been studied previously, see e.g.~\cite{MR3916109,MR3941370,MR3949269}. However, the cycle correlations of \(Y\) considered in the present paper are much stronger compared to these previous results. The second part of this paper concerns the additive bias case which is a more realistic model for machine learning applications. More precisely, we consider the random feature matrix \(Y=f(WX+B)\), where \(B\) is a rectangular rank-one Gaussian random matrix, and derive a characterization of the Stieltjes transform of the limiting spectral density. We chose \(B\) to be rank-one since for the most commonly used neural network architectures the added bias is chosen equal for each sample.~\cite{1912.00827} studied the bias case for deterministic data and i.i.d.\ Gaussian random weights, and computed the exact training error of a ridge-regularized noisy autoenconder in the high-dimensional regime. Interestingly we find that in the case of additive bias it is impossible to choose an activation function \(f\) such that the eigenvalue distribution is preserved throughout multiple layers, unlike in the bias-free case where \(\theta_2(f)=0\) yields the Marchenko-Pastur distribution in each layer. Finally, we remark that in the bias-free case our proof via the resolvent method has no significant advantage compared to the moment method, beyond requiring less combinatorics. The main advantage of the resolvent approach is that it allows to include an additive bias without much additional effort. 

\section{Model and main results}\label{section2}
We consider a random data matrix \(X \in \mathbb{R}^{n_0 \times m}\) with i.i.d.\ random variables \(X_{ij}\) with distribution \(\nu_1\) and a random weight matrix \(W \in \mathbb{R}^{n_1 \times n_0}\) with i.i.d.\ weights \(W_{ij}\) with distribution \(\nu_2\). We assume that both distributions are centred with variance \(\mathbf{E} X^2_{ij} = \sigma_x^2\) and \(\mathbf{E} W^2_{ij}= \sigma_w^2\). Moreover, we assume that the distributions \(\nu_1,\nu_2\) have finite moments of all orders\footnote{This assumption can be relaxed by a customary cut-off argument, but we refrain from doing so for simplicity.}. Since for \(1 \leq i \leq n_1\) and \(1 \leq j \leq m\) we have
\[\left (\frac{WX}{\sqrt{n_0}} \right )_{ij} = \frac{1}{\sqrt{n_0}} \sum_{k=1}^{n_0} W_{ik} X_{kj},\]
we note that in light of the central limit theorem the entries of the matrix \(\frac{WX}{\sqrt{n_0}}\) are approximately \(\mathcal{N}(0, \sigma_w^2 \sigma_x^2)\)-normally distributed random variables. Therefore, for any \(t > 0\), we have the large deviation estimate
\[\mathbf{P} \left ( \max_{i,j} \, \left |\frac{( WX)_{ij}}{\sqrt{n_0}} \right | > t \right ) \lesssim  n_0^2 \, e^{-t^2/2\sigma_w^2\sigma_x^2},\]
where we use the notation \(A \lesssim B \) as shorthand for the inequality \( A \le c B\) for some constant \(c\). Let \(f \colon \mathbb{R} \to \mathbb{R}\) be a \(C^\infty\) function with zero mean with respect to the Gaussian density of standard deviation \(\sigma_w\sigma_x\), i.e.\
\begin{equation}\label{(2.1)}
\int_\mathbb{R} f(\sigma_w \sigma_x x) \, \frac{e^{-x^2/2}}{\sqrt{2 \pi}}\text{d}x=0.
\end{equation}

We consider the random feature model generated by a single-hidden-layer neural network, 
\begin{equation}\label{(2.2)}
M = \frac{1}{m} Y Y^\ast \in \mathbb{R}^{n_1 \times n_1} \quad \text{with} \enspace Y = f \left ( \frac{WX}{\sqrt{n_0}}\right ),
\end{equation}
where the activation function \(f\) is applied entry-wise. Let \(\chi \colon \mathbb{R} \to \mathbb{R}\) be a smooth cut-off function that is equal to one for \(|x|\leq 1\) and zero for \(|x| \geq 2\). We then replace \(f\) by \(f(\cdot) \chi(\log^{-1}(n_0) \, \cdot)\). In particular, we now have that \(f\) is smooth with compact support. Moreover, for any \(l >0\) and \(n_0\) large enough, with probability \(1 - n_0^{- l}\), the singular values of \(Y\) remain the same.

We are interested in the eigenvalue density of the random matrix \(M\) in the infinite size limit. So, we assume that the dimensions of both the columns and the rows of each matrix are large and grow at the same speed, i.e.\ we introduce some positive constants \(\phi\) and \(\psi\) such that
\begin{equation}
\frac{n_0}{m} \longrightarrow \phi \enspace \text{and} \enspace \frac{n_0}{n_1} \longrightarrow \psi \quad \text{as} \enspace n_0, n_1, m \to \infty.
\end{equation}
We denote by \((\lambda_1, \dots, \lambda_{n_1})\) the eigenvalues of \(M\) and define its empirical spectral distribution by 
\(\mu_{n_1} = \frac{1}{n_1} \sum_{i=1}^{n_1} \delta_{\lambda_i}.\)
Then, as \(n_1\) grows large, the empirical distribution of eigenvalues converges in distribution to some deterministic limiting density.
\begin{thm}\label{thm2.1}
There exists a deterministic measure \(\mu=\mu_{\phi, \psi}(\theta_1,\theta_2)\) such that almost surely weakly 
\[\mu_{n_1} \longrightarrow \mu \quad \text{as} \enspace n_1 \to \infty.\]
\end{thm}

We notice that if \(m < n_1\), then \(\rank (M) = \min (n_1, m)=m\) and \(M\) has \(n_1-m\) zero eigenvalues. In this case, since \(\phi/\psi>1\), there exists an atom at \(0\) with mass \(\mu_{n_1}(0) = 1 - \psi/\phi >0\), and we have 
\[\mu_{n_1} = \frac{n_1-m}{n_1} \delta_0 +  \frac{1}{n_1} \sum_{i=1}^{n_1} \delta_{\lambda_i}.\]
Conversely, if \(n_1 < m\), the matrix \(M\) has full rank and it is invertible. Since the nonzero eigenvalues of \(Y Y^\ast\) and \(Y^\ast Y\) are the same, the limiting measure \(\mu\) of Theorem~\ref{thm2.1} turns out to be 
\[\mu = \left ( 1 - \frac{\psi}{\phi} \right )_+ \delta_0 + \tilde{\mu},\]
where \(( \cdot)_+ = \max(0, \cdot)\), and \(\tilde{\mu}\) is the limiting spectral measure of \(\frac{1}{m}Y^\ast Y\). 

We will prove that the deterministic measure \(\mu\) of Theorem~\ref{thm2.1} is characterized through a quartic self-consistent equation for the Stieltjes transform \(g(z) = \frac{1}{n_1} \Tr G(z)\) of the empirical spectral measure \(\mu_{n_1}\), where \[G(z) = \left ( M - z \right)^{-1} \in \mathbb{C}^{n_1 \times n_1}\] is the resolvent of the random matrix \(M\) and the spectral parameter \(z\) lies in the upper half plane \(\mathbb{H} = \{ z \in \mathbb{C} \, | \, \Im z \geq 0 \}\). We set 
\begin{equation}\label{theta def}
\theta_1(f) \coloneqq \int_\mathbb{R}  f^2(\sigma_w \sigma_x x) \frac{e^{-x^2/2}}{\sqrt{2 \pi}} \mathrm{d}x \quad \text{and} \quad \theta_2(f) \coloneqq \left ( \sigma_w \sigma_x \int_\mathbb{R}  f'(\sigma_w \sigma_x x) \frac{e^{-x^2/2}}{\sqrt{2 \pi}} \mathrm{d}x\right )^2.
\end{equation}
Then, the following theorem characterizes \(g\) as the solution to a quartic equation which depends only on the two parameters \(\theta_1(f)\) and \(\theta_2(f)\).
\begin{thm}\label{thm2.2}
For some \(\delta, \epsilon > 0\) and any \(z \in \mathbb{H}\) with \(\Im z > n_1^{-\frac{1}{4}+\epsilon}\), the measure \(\mu\) is characterized through the following self-consistent equation
\begin{equation}\label{self-const eq}
\begin{split}
& \left | 1 + zg  - \left ( \theta_1 - \frac{\theta_2}{\psi} (1 + zg)  \right ) g \left ( 1- \frac{\phi}{\psi} (1 + zg) \right )- \frac{\theta_2  (\theta_1 - \theta_2) }{\psi}g^2 \left ( 1- \frac{\phi}{\psi} (1 + zg) \right )^2 \right | \le n_1^{-\delta}
\end{split}
\end{equation}
almost surely.
\end{thm}

\begin{rmk}~
  \begin{enumerate}[itemsep=0em,topsep=0em,label=(\roman*),labelindent=.8em,labelwidth=1em, labelsep*=.5em, leftmargin =!]
    \item We obtain an analogous result for complex feature and weight matrices \(W,X\), c.f.\ Remark~\ref{remark complx}.
    \item Note that the quartic self-consistent equation~\eqref{self-const eq} may not have a unique solution such that \(\Im g(z) >0\). However, it has a unique solution which is analytic in the upper half-plane and satisfies \(g(z)\sim -1/z\) for large \(\lvert z\rvert\).
    \item Since the resolvent itself satisfies \(\Tr G(z)/n_1\sim-1/z\) for large \(\abs{z}\) and is analytic in the upper half-plane, by continuity Theorem~\ref{thm2.2} implies that \(g(z)\) is approximately given by a properly chosen solution of~\eqref{self-const eq}. Then, the limiting spectral measure itself can be recovered via the Stieltjes inversion formula,
    \[\mu(\lambda) = \lim_{\epsilon \to 0^{+}} \frac{1}{\pi} \Im g(\lambda + i \epsilon),\]
    and Theorem~\ref{thm2.1} follows from Theorem~\ref{thm2.2}.
    \item It follows from the self-consistent equation~\eqref{self-const eq} that the limiting spectral measure \(\tilde{\mu}\) is absolutely continuous w.r.t.\ the Lebesgue measure, and therefore so is \(\mu\) away from zero. Moreover, for large \(z\), equation~\eqref{self-const eq} has real solutions and thus via Stieltjes inversion the limiting measure \(\mu\) is compactly supported.
  \end{enumerate}
\end{rmk}

\begin{rmk}
It should be noted that Theorem~\ref{thm2.1} and Theorem~\ref{thm2.2} were proven in~\cite{PW17,1904.03090} under different assumptions and with a different method. The result in~\cite{PW17} was obtained for i.i.d.\ Gaussian features and weights, whereas~\cite{1904.03090} extends the result to the case where both the inputs and the random weights have sub-Gaussian tails but are not necessarily Gaussian.
\end{rmk}

Observing equation~\eqref{self-const eq}, we note that if \(\theta_2(f) = 0\), then the limiting measure \(\mu\) is exactly the Marchenko-Pastur \(\mu_{MP}\) distribution with parameter \(\phi/\psi\). Indeed, in this case, \(g(z)\) approximately satisfies the quadratic equation
\begin{equation}\label{MP eq}
1 + \left (z + \theta_1(f) \left ( \frac{\phi}{\psi} -1 \right ) \right ) g(z)+ \theta_1(f) \frac{\phi}{\psi} z g(z)^2 \approx 0,
\end{equation}
which corresponds to the self-consistent equation satisfied by the Stieltjes transform of \(\mu_{MP}\)~\cite{MR208649}. As discussed in the introduction, this consideration is relevant when studying multilayer networks. Pennington and Worah~\cite{PW17} conjectured that the asymptotic spectral distribution is preserved through multiple layers only by activation functions with \(\theta_2(f)=0\) and is given by the Marchenko-Pastur distribution in each layer. Benigni and P\'{e}ch\'{e}~\cite{1904.03090} then proved this conjecture for bounded activation functions satisfying \(\theta_2(f)=0\). Moreover, if \(\theta_1(f) = \theta_2(f)\), then equation~\eqref{self-const eq} becomes cubic. In particular, the equality \(\theta_1(f) = \theta_2(f)\) holds if and only if \(f\) is a linear function (for more details, we refer to the supplementary material in~\cite{PW17}). In this case, \(M = \frac{1}{m}YY^\ast\) with \(Y=WX\), and thus the limiting measure \(\mu\) corresponds to the limiting spectral distribution of a product Wishart matrix. The spectral density for matrices of this type has been computed in~\cite{1401.7802, MR3251989}.

\begin{figure}[htbp]
  \centering
  \begin{tikzpicture}
    \begin{axis}[width=7.5cm,height=4cm, ymin=0,xmax=1.2,title={$f(x)=\tanh(x)$},ylabel={Without bias}]
        \addplot[ybar,bar width=1.2/40,fill=col3,bar shift=0.0] table[col sep=comma,x index=0,y index=1] {hist_tanh.csv};
        \addplot[thick,mark=none] table[col sep=comma,x index=0,y index=1] {curve_tanh.csv};
    \end{axis}
\end{tikzpicture}
\begin{tikzpicture}
    \begin{axis}[width=7.5cm,height=4cm, ymin=0,xmax=80,title={$f(x)=x^3$}]
        \addplot[ybar,bar width=80/40,fill=col3,bar shift=0.0] table[col sep=comma,x index=0,y index=1] {hist_cube.csv}; 
        \addplot[thick,mark=none] table[col sep=comma,x index=0,y index=1] {curve_cube.csv};
    \end{axis}
\end{tikzpicture}\\
\begin{tikzpicture}
  \begin{axis}[width=7.5cm,height=4cm, ymin=0,xmax=1.2,ylabel={With bias}]
      \addplot[ybar,bar width=1.2/40,fill=col3,bar shift=0.0] table[col sep=comma,x index=0,y index=2] {hist_tanh.csv};
      \addplot[thick,mark=none] table[col sep=comma,x index=0,y index=2] {curve_tanh.csv};
  \end{axis}
\end{tikzpicture}
\begin{tikzpicture}
  \begin{axis}[width=7.5cm,height=4cm, ymin=0,xmax=80]
      \addplot[ybar,bar width=80/40,fill=col3,bar shift=0.0] table[col sep=comma,x index=0,y index=2] {hist_cube.csv};
      \addplot[thick,mark=none] table[col sep=comma,x index=0,y index=2] {curve_cube.csv};
  \end{axis} 
\end{tikzpicture}
  \caption{We present the eigenvalue histogram of the covariance matrix \(YY^\ast\) for a single random realisation together with the theoretical limit from Theorems~\ref{thm2.1} and~\ref{thm2.6} for the functions \(f(x)=\tanh(x)\) and \(f(x)=x^3\) with and without additive bias. We note that the presence of an additive bias can both increase or decrease the largest singular value. The numerical experiments were conducted for the parameters \(n_1=3000\), \(\phi=\sigma_x=\sigma_w=1\), \(\psi=5\) (left) or \(\psi=2\) (right), and \(\sigma_b=0\) (top) or \(\sigma_b=0.25\) (bottom).}
 \end{figure}
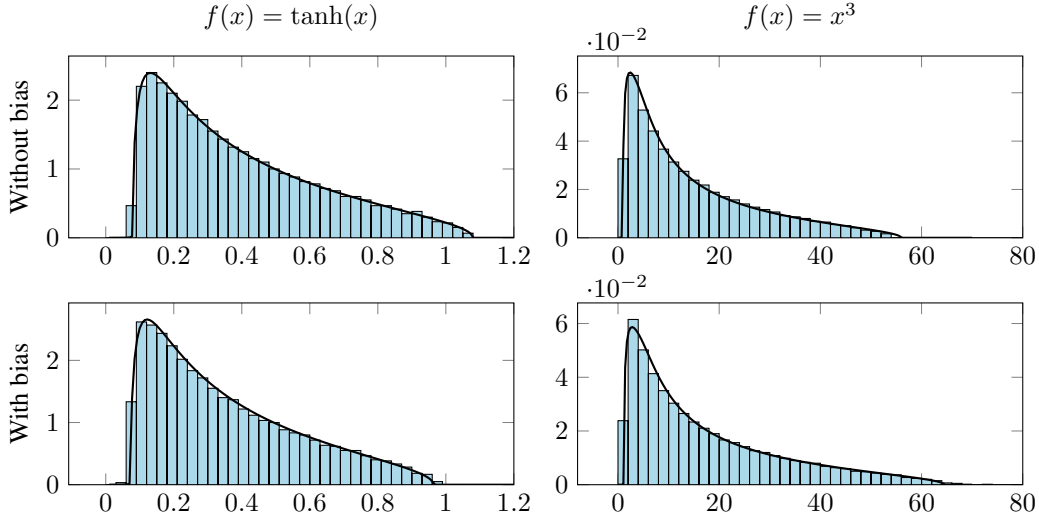
\subsection{Additive bias case}
The previous model can be generalised by adding random biases. In neural networks, the bias is an additional parameter that allows the model to better fit the given data. In this case, for each input data \(x \in \mathbb{R}^{n_0}\), a bias vector \(b \in \mathbb{R}^{n_1}\) is added to the vector \(Wx \in \mathbb{R}^{n_1}\). We then apply a non-linear function \(f \colon \mathbb{R} \to \mathbb{R}\) in an element-wise fashion to its vector arguments \(Wx+b\) in order to obtain \(n_1\) random features \(f(Wx + b) \in \mathbb{R}^{n_1}\).

We consider a random bias matrix \(B \in \mathbb{R}^{n_1 \times m}\) of i.i.d.\ Gaussian random variables \(B_{ij}=B_i\) with zero mean and variance \(\mathbf{E} B^2_i  = \sigma_b^2\). Note that the random matrix \(B\) has rank 1. Let \(X \in \mathbb{R}^{n_0 \times m}\) and \(W \in \mathbb{R}^{n_1 \times n_0}\) be random matrices with i.i.d.\ entries, defined as before. Moreover, let \(f \colon \mathbb{R} \to \mathbb{R}\) be a \(C^\infty\) function satisfying
\begin{equation}\label{eq:2.7}
\int_\mathbb{R} f \left (\sqrt{\sigma_w^2 \sigma_x^2 + \sigma_b^2}  \, x \right) \, \frac{e^{-x^2/2}}{\sqrt{2 \pi}} \text{d}x=0.
\end{equation}
Just as before, without loss of generality, upon replacing \(f\) by \(f(\cdot) \chi(\log^{-1}(n_0) \, \cdot)\), we may assume that \(f\) is a smooth function with compact support. We then define the random matrix \(M\) by
\begin{equation}\label{(2.8)} 
M = \frac{1}{m} Y Y^\ast \in \mathbb{R}^{n_1 \times n_1} \quad \text{with} \enspace Y = f \left ( \frac{WX}{\sqrt{n_0}} + B\right ),
\end{equation}
where \(f\) is applied entry-wise. We introduce the parameter 
\begin{equation*}
  \widetilde \sigma \coloneqq \sqrt{\frac{\sigma_w^2\sigma_x^2(\sigma_w^2\sigma_x^2+2\sigma_b^2)}{\sigma_w^2\sigma_x^2+\sigma_b^2}},
\end{equation*}
and we define the following integral parameters:

\begin{equation}\label{eq:2.9}
\begin{split}
\theta_1(f) &\coloneqq \int_\mathbb{R}  f^2 \left ( \sqrt{\sigma^2_w \sigma^2_x + \sigma_b^2} x \right ) \frac{e^{-x^2/2}}{\sqrt{2 \pi}} \mathrm{d}x,\\
\theta_{1,b}(f) &\coloneqq \frac{1}{2 \pi \widetilde \sigma \sqrt{\sigma_w^2 \sigma_x^2 + \sigma_b^2}} \int_{\mathbb{R}^2}  f(x_1) f(x_2) \exp \left ( - \frac{x_1^2 + x_2^2}{2  \widetilde{\sigma}^2} +\frac{ \sigma_b^2 x_1 x_2}{ \widetilde{\sigma}^2 (\sigma_w^2 \sigma_x^2 + \sigma_b^2)}\right ) \mathrm{d}\bm{x},\\
\theta_{2}(f) &\coloneqq \frac{\sigma_w \sigma_x}{2 \pi \widetilde \sigma \sqrt{\sigma_w^2 \sigma_x^2 + \sigma_b^2}} \int_{\mathbb{R}^2}  f'(x_1) f'(x_2) \exp \left ( - \frac{x_1^2 + x_2^2}{2  \widetilde{\sigma}^2} +\frac{ \sigma_b^2 x_1 x_2}{  \widetilde{\sigma}^2 (\sigma_w^2 \sigma_x^2 + \sigma_b^2)}\right )  \mathrm{d}\bm{x}.
\end{split}
\end{equation}

We can now state the analogue of Theorem~\ref{thm2.2} in the additive bias case. In particular, the following theorem shows that the normalized trace of the resolvent of \(M\) approximately satisfies the self-consistent equation~\eqref{self-const eq} with parameters given by~\eqref{eq:2.9}.

\begin{thm}\label{thm2.6}
The Stieltjes transform \(g\) satisfies~\eqref{self-const eq} with parameters given by~\eqref{eq:2.9}, where \(\theta_1(f)\) is replaced by \(\theta_1(f)-\theta_{1,b}(f)\). Moreover, there exists a single outlier eigenvalue \(\lambda_{\max} = n_1 \theta_{1,b}(1+\mathcal O(n_1^{-1/2}))\) of \(M\) that is separated from the support of the rest of the spectrum. 
\end{thm}

We remark that the parameters \(\theta_{1,b}(f),\theta_2(f)\) can be alternatively expressed as infinite series, directly demonstrating that for \(\sigma_b\ne 0\) and non-trivial \(f\) both coefficients are strictly positive, \(\theta_{1,b}(f),\theta_2(f)>0\). For notational implicitly, we introduce the Hermite inner product 
\begin{equation*}
  \braket{f,g}_{\mathrm{He}} \coloneqq \frac{1}{\sqrt{2\pi}} \int_\mathbb{R} f(x)g(x) e^{-x^2/2}\dif x.
\end{equation*}

\begin{rmk} \label{rmk2.6}
  We have 
  \begin{equation}
    \begin{split}
      \theta_{1,b}(f) &=  \frac{\widetilde\sigma}{\sqrt{\sigma_w^2\sigma_x^2+\sigma_b^2}} \sum_{k\ge 0} \frac{1}{k!} \Bigl( \frac{\sigma_b^2}{\sigma_w^2\sigma_x^2+\sigma_b^2} \Bigr)^k \braket{x^k,f(\widetilde \sigma \cdot)}_{\mathrm{He}}^2  \\
      \theta_{2}(f) &=  \frac{\sigma_w^2\sigma_x^2\widetilde\sigma}{\sqrt{\sigma_w^2\sigma_x^2+\sigma_b^2}} \sum_{k\ge 0} \frac{1}{k!} \Bigl( \frac{\sigma_b^2}{\sigma_w^2\sigma_x^2+\sigma_b^2} \Bigr)^k \braket{x^k,f'(\widetilde \sigma \cdot)}_{\mathrm{He}}^2
    \end{split}
  \end{equation}
  and therefore \(\theta_{1,b}(f)=0\), \(\sigma_b\ne 0\) implies that \(f(\widetilde\sigma\cdot)\) is orthogonal to Hermite polynomials of any order, and consequently \(f\equiv 0\). Similarly, \(\theta_{2}(f)=0\), \(\sigma_b\ne 0\) implies that \(f\equiv \mathrm{const}\).
\end{rmk}

\subsection{Multiple layers}
In~\cite{PW17} it was observed empirically that in the bias-free case activation functions with \(\theta_2(f)=0\) have the remarkable property that for multiple layers
\begin{equation}
   Y^{(l+1)} \coloneqq f(W^{(l)}Y^{(l)}), \qquad Y^{(0)} \coloneqq X
\end{equation}
the singular value distributions of \(Y^{(1)},Y^{(2)},\ldots\) all asymptotically agree (up to scaling) with the probability distribution \(\mu(\theta_1,\theta_2)=\mu(\theta_1,0)\) from Theorem~\ref{thm2.1}. This observation is very natural from our point of view since we find that \(Y^{(1)}\) is approximately an i.i.d.\ random matrix if \(\theta_2(f)=0\), c.f.\ Proposition~\ref{prop:3.5} below. 

An interesting corollary of our Theorem~\ref{thm2.6} is that a similar isospectral property \emph{cannot} be ensured for the case of additive bias
\begin{equation}
  Y^{(l+1)}:=f(W^{(l)}Y^{(l)}+B^{(l)}), \qquad Y^{(0)} \coloneqq X.
\end{equation}
Indeed, in light of Remark~\ref{rmk2.6}, for \(\sigma_b \ne 0\) we have \(\theta_{1,b}(f),\theta_2(f)>0\) for all activation functions \(f\), and therefore already the random matrix \(Y^{(1)}\) necessarily has leading order correlations, c.f.\ Proposition~\ref{prop:4.1} below. Hence, convergence of the spectral density to the solution of~\eqref{self-const eq} is not expected beyond the first layer. In Fig.~\ref{wasserstein} we test this result experimentally and choose the activation function \(f(x)=c_1|x|-c_2\) with \(c_1,c_2\) such that~\eqref{(2.1)} is satisfied and \(\theta_1(f)=1\). We find that in the bias-free case (left), irrespective of the network depth, the eigenvalues of the covariance matrix \(Y^{(l)} (Y^{(l)})^\ast\) converge to their theoretical limit from Theorem~\ref{thm2.1}, exactly as in~\cite[Fig. 1]{PW17}\footnote{In the notation of~\cite{PW17}, \(f=f_1\).}. In the case of an additive bias (right), no such convergence is observed, and this provides empirical evidence of our result.

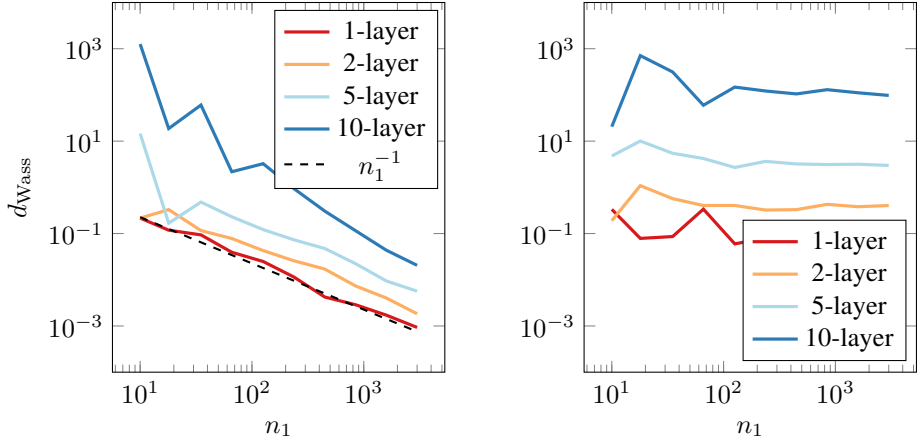
\begin{figure}[htbp]
  \centering
  \pgfplotsset{try min ticks=3}  
  \begin{tikzpicture}
      \begin{axis}[width=6cm,height=6.5cm,%
          xmode=log,ymode=log,ymin=0.0001,ymax=10000,xlabel={$n_1$},ylabel={$d_{\mathrm{Wass}}$}
          ]
          \addplot[col1,very thick,mark=none] table[col sep=comma,x index=0,y index=1] {wasserstein_no_bias.csv};
          \addlegendentry{1-layer}
          \addplot[col2,very thick,mark=none] table[col sep=comma,x index=0,y index=2] {wasserstein_no_bias.csv};
          \addlegendentry{2-layer}
          \addplot[col3,very thick,mark=none] table[col sep=comma,x index=0,y index=3] {wasserstein_no_bias.csv};
          \addlegendentry{5-layer}
          \addplot[col4,very thick,mark=none] table[col sep=comma,x index=0,y index=4] {wasserstein_no_bias.csv};
          \addlegendentry{10-layer}
          \addplot [thick, domain=10:3000, samples=19,dashed] {2.26*(x)^(-1)};
          \addlegendentry{\(n_1^{-1}\)}
      \end{axis}
  \end{tikzpicture}
  \hspace{0.7cm}
  \begin{tikzpicture}
    \begin{axis}[width=6cm,height=6.5cm,%
        ymin=0.0001,ymax=10000,%
        xmode=log,ymode=log, legend pos=south east,xlabel={$n_1$}
        ]
        \addplot[col1,very thick,mark=none] table[col sep=comma,x index=0,y index=1] {wasserstein_bias.csv};
        \addlegendentry{1-layer}
        \addplot[col2,very thick,mark=none] table[col sep=comma,x index=0,y index=2] {wasserstein_bias.csv};
        \addlegendentry{2-layer}
        \addplot[col3,very thick,mark=none] table[col sep=comma,x index=0,y index=3] {wasserstein_bias.csv};
        \addlegendentry{5-layer}
        \addplot[col4,very thick,mark=none] table[col sep=comma,x index=0,y index=4] {wasserstein_bias.csv};
        \addlegendentry{10-layer}
    \end{axis}
\end{tikzpicture}
  \caption{For randomly generated neural networks of varying depth and width, we compute the Wasserstein distance \(d_{\mathrm{Wass}}\) between the empirical eigenvalue density of the covariance matrix \(Y^{(l)} (Y^{(l)})^\ast\) to the distribution \(\mu\) from Theorem~\ref{thm2.1} for the activation function \(f(x)=c_1|x|-c_2\). In the bias-free case (left), the Wasserstein distance decays as the inverse of the network width, while in the case of an additive bias (right) no convergence can be observed. The numerical experiments were conducted for the parameters \(\phi=\sigma_x=\sigma_w=1\), \(\psi=2\) and \(\sigma_b=0\) (left) or \(\sigma_b=0.5\) (right).}\label{wasserstein} 
\end{figure} 

The spectrum of the covariance matrix \(Y^{(l)}(Y^{(l)})^\ast\) reflects the distortion of input data through the network and highly skewed distributions indicate poor conditioning which may impede learning performance~\cite{PW17}. \emph{Batch normalization} seeks to remedy the distortion by normalising by the trace of the covariance matrix \(Y^{(l)}(Y^{(l)})^t\) in each layer. In~\cite{PW17} it was suggested that choosing activation functions with \(\theta_2(f)=0\), i.e.\ functions which naturally preserve the singular value distribution, may serve as an alternative method of tuning networks for fast optimisation. Our result indicates that in the case of additive bias this alternative is not present. However, batch normalization seems to help stabilising the singular value distribution also in the additive bias case, c.f.\ Fig.~\ref{batch figure}.

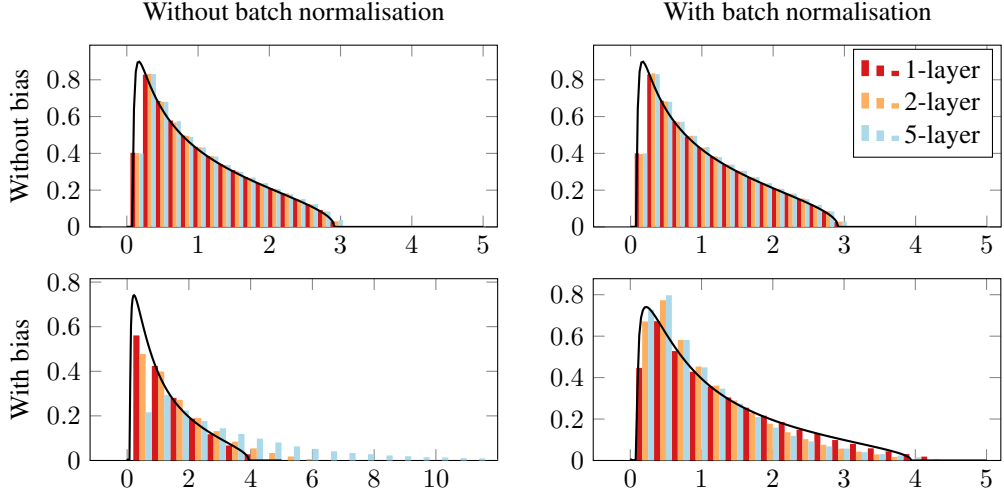
\begin{figure}[htbp]
\centering\noindent
\begin{tikzpicture}
    \begin{axis}[width=7cm,height=4cm, ymin=0,xmax=5.2,title={Without batch normalisation},ylabel={Without bias}]
        \addplot[ybar,bar width=3.5/20/3,fill=col1,bar shift=0.0,draw=col1] table[col sep=comma,x index=0,y index=1] {hists_nobatch_nobias.csv};
        \addplot[ybar,bar width=3.5/20/3,fill=col2,bar shift=3.5/20/3,draw=none] table[col sep=comma,x index=0,y index=2] {hists_nobatch_nobias.csv};
        \addplot[ybar,bar width=3.5/20/3,fill=col3,bar shift=3.5/20/3*2,draw=none] table[col sep=comma,x index=0,y index=3] {hists_nobatch_nobias.csv};
        \addplot[thick,mark=none] table[col sep=comma,x index=0,y index=1] {f1_curve.csv};
    \end{axis}
\end{tikzpicture}\; \hspace{0.3cm}
\begin{tikzpicture}
  \begin{axis}[width=7cm,height=4cm, ymin=0,xmax=5.2,title={With batch normalisation}]
    \addplot[ybar,bar width=3.5/20/3,fill=col1,bar shift=0.0,draw=none,ybar legend] table[col sep=comma,x index=0,y index=1] {hists_batch_nobias.csv};
    \addlegendentry{1-layer}
    \addplot[ybar,bar width=3.5/20/3,fill=col2,bar shift=3.5/20/3,draw=none,ybar legend] table[col sep=comma,x index=0,y index=2] {hists_batch_nobias.csv};
    \addlegendentry{2-layer}
    \addplot[ybar,bar width=3.5/20/3,fill=col3,bar shift=3.5/20/3*2,draw=none,ybar legend] table[col sep=comma,x index=0,y index=3] {hists_batch_nobias.csv};
    \addlegendentry{5-layer}
    \addplot[thick,mark=none] table[col sep=comma,x index=0,y index=1] {f1_curve.csv};
  \end{axis}
\end{tikzpicture}\\\noindent 
\begin{tikzpicture}
    \begin{axis}[width=7cm,height=4cm, ymin=0,xmax=12,ylabel={With bias},xtick={0,2,4,6,8,10}]
          \addplot[ybar,bar width=12/20/3,fill=col1,bar shift=0.0,draw=none] table[col sep=comma,x index=0,y index=1] {hists_nobatch_bias.csv};
          \addplot[ybar,bar width=12/20/3,fill=col2,bar shift=12/20/3,draw=none] table[col sep=comma,x index=0,y index=2] {hists_nobatch_bias.csv};
          \addplot[ybar,bar width=12/20/3,fill=col3,bar shift=12/20/3*2,draw=none] table[col sep=comma,x index=0,y index=3] {hists_nobatch_bias.csv};
          \addplot[thick,mark=none] table[col sep=comma,x index=0,y index=1] {f1_curve_bias.csv};
      \end{axis}
  \end{tikzpicture}\; \hspace{0.3cm}
  \begin{tikzpicture}
    \begin{axis}[width=7cm,height=4cm, ymin=0,xmax=5.2]
        \addplot[ybar,bar width=5/20/3,fill=col1,bar shift=0.0,draw=none] table[col sep=comma,x index=0,y index=1] {hists_batch_bias.csv};
        \addplot[ybar,bar width=5/20/3,fill=col2,bar shift=5/20/3,draw=none] table[col sep=comma,x index=0,y index=2] {hists_batch_bias.csv};
        \addplot[ybar,bar width=5/20/3,fill=col3,bar shift=5/20/3*2,draw=none] table[col sep=comma,x index=0,y index=3] {hists_batch_bias.csv};
        \addplot[thick,mark=none] table[col sep=comma,x index=0,y index=1] {f1_curve_bias.csv};
    \end{axis}
\end{tikzpicture}
  \caption{We present the eigenvalue distribution of neural networks of varying depth and in the presence/absence of both bias and batch normalization for the activation function \(f(x)=c_1|x|-c_2\). In the bias-free case, batch normalisation has no effect on the spectral stability, and throughout the network the theoretical distribution from Theorem~\ref{thm2.1} matches the actual eigenvalue distribution of the covariance matrix \(Y^{(l)}(Y^{(l)})^\ast\) well. In the case of an additive bias, the single-layer spectral density matches the theoretical limit from Theorem~\ref{thm2.6} to high accuracy. However, for multiple layers the spectral density diverges without additional batch normalization. Batch normalization alleviates the divergence, however the actual eigenvalue distribution deviates from the theoretical limit from Theorem~\ref{thm2.6}. The numerical experiments were conducted for the parameters \(n_1=3000\), \(\phi=\sigma_x=\sigma_w=1\), \(\psi=2\) and \(\sigma_b=0\) (top) or \(\sigma_b=0.5\) (bottom). Here we used batch normalisation of the form \(Y^{(l)}\mapsto c Y^{(l)}\) after each layer, choosing \(c\) to ensure unit empirical variance.}\label{batch figure}  
\end{figure}

\section{Outline of proof of Theorems~\ref{thm2.2} and~\ref{thm2.6}} \label{proof}

The proof of both Theorem~\ref{thm2.2} and~\ref{thm2.6} can be broken into two distinct parts. The first step is to show that \(Y=f \left ( \frac{WX}{\sqrt{n_0}} \right ) \in \mathbb{R}^{n_1 \times m}\) can be viewed as a correlated random matrix with \textit{cycle correlations}, c.f.~Propositions~\ref{prop:3.5} and~\ref{prop:4.1} below. The second step is to prove the global law for the random matrix \(M=\frac{1}{m}Y Y^\ast\) with the cycle correlations. In the following, we will sketch the derivation of the self-consistent equation. A more detailed proof is provided in the supplementary material.

The key idea is to use a multivariate cumulant expansion formula. Cumulants of a random vector \(\bm{X} = (X_1, \dots, X_n)\) can be defined in a combinatorial way by
\begin{equation}\label{cum comb}
\kappa(X_1, \dots, X_n) =  \sum_{\pi} (-1)^{| \pi | -1} (|\pi| -1)! \prod_{B \in \pi} \mathbf{E} \left ( \prod_{i \in B} X_i \right ),
\end{equation}
where the sum runs over all partitions \(\pi\) of the set \([n]=\{1, \dots, n \}\), the product runs over the blocks \(B\) of the partition \(\pi\), and \(|\pi|\) is the number of blocks in the partition. The following expansion is commonly referred to as a cumulant expansion and generalises the Gaussian integration by parts. In the context of random matrix theory, the usefulness of this expansion was first observed in~\cite{MR1411619} and later revived in~\cite{MR3678478, MR3800833}. A proof of the following lemma is provided in Appendix~\ref{appendixA} for completeness. 

\begin{lem}[Cumulant expansion]\label{lemma3.1}
If \(\bm{X} = (X_1, \dots, X_n)\) is a random vector with finite moments of all orders, then
\[\mathbf{E} X_1 f(\bm{X}) = \sum_{l \geq 1} \sum_{i_1, \dots, i_l} \frac{\kappa(X_1, X_{i_1}, \dots, X_{i_l})}{l!} \: \mathbf{E} \partial_{i_1} \cdots \partial_{i_l} f(\bm{X}),\]
where \(f\colon \mathbb{R}^n \to \mathbb{R}\) is smooth. 
\end{lem}

We start with the defining identity of the resolvent, \(\mathbf{1}_{n_1} + z G = MG\), where \(\textbf{1}_{n_1}\) denotes the \(n_1 \times n_1 \) identity matrix, and we compute its average trace:
\begin{equation} \label{eq:3.2}
1 + zg = \frac{1}{n_1} \Tr \frac{Y Y^\ast G}{m} = \frac{1}{n_1} \sum_{i=1}^{n_1} \sum_{j=1}^{m} Y_{ij} \left (\frac{Y^\ast G}{m} \right )_{ji},
\end{equation}
where \(g(z) = \frac{1}{n_1} \Tr (M-z \mathbf{1}_{n_1})^{-1}\) is the normalized trace of the resolvent of \(M\). Since the random variable \(\left ( Y^\ast G \right )_{ji}\) can be seen as a function of \(Y_{ij}\), we can take the expectation on both sides of~\eqref{eq:3.2} and apply Lemma~\ref{lemma3.1}:
\begin{equation}\label{eq:3.3}
1 + z \, \mathbf{E} g = \frac{1}{n_1} \sum_{k \geq 1} \sum_{i_1, \dots, i_{2k}} \frac{\kappa(Y_{i_1 i_2}, Y_{i_3 i_4}, \dots, Y_{i_{2k-1} i_{2k}})}{(k-1)!} \: \mathbf{E} \partial_{Y_{i_3 i_4}} \cdots  \partial_{Y_{i_{2k-1} i_{2k}}} \left (\frac{Y^\ast G}{m} \right )_{i_2 i_1}.
\end{equation} 
The main goal now is to show that \(Y\) can be viewed as a random matrix with cycle correlations given as in the Propositions~\ref{prop:3.5} and~\ref{prop:4.1} below: Prop.~\ref{prop:3.5} refers to the bias-free case and Prop.~\ref{prop:4.1} to the additive bias case. We postpone the proof of both propositions to Subsections~\ref{subsection3.4} and~\ref{subsection4.2}, resp.
\begin{prop}[Correlation structure without bias]\label{prop:3.5}
The random matrix \(Y\) defined by~\eqref{(2.2)} has joint cumulants given by 
\begin{equation}\label{eq:3.4}
        \begin{split}
            \kappa(Y_{i_1 i_2})&= \mathcal{O}(n_0^{-1/2}),\\
            \kappa(Y_{i_1 i_2},Y^\ast_{i_2 i_1})&\approx \theta_1(f),\\
             \kappa(Y_{i_1i_2},Y^\ast_{i_2i_3}, Y_{i_3 i_4}, \ldots,Y^\ast_{i_{2k} i_1}) & \approx \theta_2(f)^k n_0^{1-k},\quad k>1
        \end{split}
\end{equation}
where \(i_1, \dots, i_{2k}\) are all distinct, and we write \(X\approx Y\) as a shorthand notation for \(X=Y(1+\mathcal O(n_0^{-1/2}))\).
\end{prop}
\begin{prop}[Correlation structure with bias]\label{prop:4.1}
The random matrix \(Y\) defined by~\eqref{(2.8)} has joint cumulants given by 
  \begin{equation}\label{eq:4.1}
    \begin{split}
            \kappa(Y_{i_1 i_2})&= \mathcal  O(n_0^{-1/2}),\\
            \kappa(Y_{i_1 i_2}, Y_{i_2 i_1}^\ast)& \approx \theta_1(f) ,\\
             \kappa(Y_{i_1 i_2}, Y_{i_3 i_1}^\ast) & \approx \theta_{1,b}(f) \\
              \kappa(Y_{i_1i_2},Y^\ast_{i_2i_3},Y_{i_3 i_4},\ldots,Y^\ast_{i_{2k} i_1})& \approx  \theta_2(f)^k n_0^{1-k},\quad k>1
        \end{split}
 \end{equation}
where \(i_1, \dots, i_{2k}\) are all distinct.
\end{prop}

Applying Propositions~\ref{prop:3.5} and~\ref{prop:4.1} to~\eqref{eq:3.3}, computing the partial derivatives and doing some bookkeeping, we get the desired equation~\eqref{self-const eq} as \(n_0, n_1, m \to \infty\). To complete the proofs of Theorems~\ref{thm2.2} and~\ref{thm2.6}, one has to show the concentration of \(g\) around \(\mathbf{E} g\), as stated in the following lemma.

\begin{lem}\label{lem3.7}
For the random matrix \(M=\frac{1}{m}YY^\ast\) and a complex number \(z \in \mathbb{H}\) such that \(\Im z > n_1^{-\frac{1}{4} + \epsilon}\), for some \(\epsilon > 0\), it holds that  
\begin{equation}
  \mathbf E_W \lvert g(z)-\mathbf E_W g(z)\rvert^4 \lesssim \frac{1}{n_1^2 (\Im z)^4}
\end{equation}
with high probability in \(X\), and analogously 
\begin{equation}
  \mathbf E_X \lvert g(z)-\mathbf E_X g(z)\rvert^4 \lesssim \frac{1}{n_1^2 (\Im z)^4}
\end{equation}
with high probability in \(W\), where \(\mathbf{E}_X\) (resp.\ \(\mathbf{E}_W\)) is the expectation in the \(X\)-space (resp.\ \(W\)-space).
\end{lem} 
The proof of this lemma relies on a standard argument (e.g.\ see the proof of the concentration inequality in~\cite[Subsection 3.3.2]{MR2567175}) and is given in Appendix~\ref{appendixC}. 

\section{Conclusion}\label{conclusion}
In this paper, we analysed the singular value distribution of fully random neural networks and found that in the case of additive biases it is impossible to achieve isospectrality by tuning the activation function. In addition, we showed that the resolvent method from random matrix theory also applies to the neural network analysis, despite the non-linearities and we expect that this robust method will prove to be useful in contexts where the conventionally used moment method becomes intractable. 

\section*{Broader impact}\label{broader impact}
Our result is a purely theoretical one for fully random features, weights and biases. Therefore, we do not expect our contribution to have ethical concerns or adverse future societal consequences.

\begin{ack}
D.\ Schröder would like to thank L.\ Benigni for illuminating discussions on the subject and both authors would like to thank him for his helpful comments on an early version of this preprint. Both authors thank the referees for their careful reading of our manuscript. This work was carried out when the first author was a research assistant at ETH Zurich in the group of W.\ Werner. The second author is supported by Dr.\ Max R\"ossler, the Walter Haefner Foundation and the ETH Z\"urich Foundation. 
\end{ack}

\section*{References}

\printbibliography[heading=none]

\newpage
\appendix

\section{Proof of Theorem~\ref{thm2.2}}\label{main proof}

\subsection{Derivation of the self-consistent equation}\label{subsection3.3}
We start from~\eqref{eq:3.3} and rely on the following power counting principles: Each derivative provides a smallness-factor of \(1/\sqrt{m}\) because \(G\) is a function of \(Y/\sqrt{m}\) and \(Y^\ast/\sqrt{m}\), while each independent summation costs a factor of \(n_1\sim m\). However, we cannot have too many independent summations for if any index appears only once in the cumulant, then the latter vanishes identically by the independence property of cumulants. For example, if \(i_2,\ldots,i_{2k}\ne i_1\), then the random variables \(Y_{i_3i_4},\ldots,Y_{i_{2k-1}i_{2k}}\) are independent of \(Y_{i_1i_2}\) in the probability space of the random variables \(\bigl\{w_{i_1a}\bigr\}_{a=1}^{n_0}\) conditioned on the remaining random variables. By the law of total expectation and the independence property it follows that 
\[\kappa(Y_{i_1i_2},\ldots,Y_{i_{2k-1}i_{2k}})=0\]
in this case. Thus we only need to sum over those cumulants in which each \(W\)- and \(X\)-index appears at least twice (we call \(i\) the \(W\)-index of \(Y_{ij},Y_{ji}^\ast\) and \(j\) the \(X\)-index). In the extreme case where each \(W\)- and \(X\)-index appears exactly twice, we either have a single cycle, or a union of cycles on disjoint index sets. In the latter case the cumulant vanishes identically by the independence property. In the former case, for a cycle of length \(2k\) there are \(k\) indices each, we obtain a factor of \(n_1^{-1}\) from the normalised sum, a factor of \(m^{-2k/2}=m^{-k}\) from the derivatives, a factor of \(n_1^k m^k\) from the summations, and finally a factor of \(n_0^{1-k}\) from the cumulant in Proposition~\ref{prop:3.5}, i.e.
\[ \frac{1}{n_1} \frac{1}{m^k}  n_1^k m^k n_0^{1-k} \sim 1\]  
and the power counting is neutral. On the contrary, when some index appears three times, the overall power counting described above is smaller by a factor of \(1/\sqrt{m}\), and thus negligible to leading order. In particular this argument shows that cycles of odd length only negligible as they cannot arise on indices in which each \(W\)- and \(X\)-index appears exactly twice. 

Thus, together with~Proposition~\ref{prop:3.5} we have (recalling that the shorthand notation \(\approx\) indicates equalities up to an error of \(n_0^{-1/2}\))
\begin{equation}\label{eq:3.7a}
\begin{split}
1 + z \, \mathbf{E} g &= \frac{1}{n_1 m} \sum_{k \geq 1} \sum_{i_1, \dots, i_{2k}} \frac{\kappa(Y_{i_1 i_2}, Y_{i_3 i_4}, Y_{i_5 i_6}, \dots, Y_{i_{2k-1} i_{2k}})}{(k-1)!} \: \mathbf{E} \partial_{Y_{i_3 i_4}} \cdots  \partial_{Y_{i_{2k-1} i_{2k}}} \left ( Y^\ast G\right )_{i_2 i_1}\\
&\approx \frac{1}{n_1 m} \sum_{k \geq 1} \sum_{i_1, \dots, i_{2k}}^\ast \kappa(Y_{i_1 i_2}, Y^\ast_{i_2 i_3}, Y_{i_3 i_4}, \dots, Y^\ast_{i_{2k} i_1}) \: \mathbf{E} \partial_{Y_{i_3 i_4}} \cdots  \partial_{Y_{i_{2k-1} i_{2k}}} \left ( Y^\ast G\right )_{i_2 i_1}\\
& = \frac{1}{n_1 m} \sum_{i_1, i_2}^\ast \kappa(Y_{i_1 i_2}, Y^\ast_{i_2 i_1}) \: \mathbf{E} \partial_{Y^\ast_{i_2 i_1}} \left ( Y^\ast G \right )_{i_2 i_1} \\
& \quad + \frac{1}{n_1 m} \sum_{k \geq 2} \sum_{i_1, \dots, i_{2k}}^\ast  \kappa(Y_{i_1 i_2}, Y^\ast_{i_2 i_3}, Y_{i_3 i_4}, \dots, Y^\ast_{i_{2k} i_1}) \: \mathbf{E}\partial_{Y^\ast_{i_2 i_3}} \cdots \partial_{Y^\ast_{i_{2k} i_1}}  \left (Y^\ast G \right )_{i_2 i_1}\\
&\approx \frac{\theta_1}{n_1 m} \sum_{i_1, i_2}^\ast \mathbf{E} \partial_{Y^\ast_{i_2 i_1}} \left (Y^\ast G \right )_{i_2 i_1} + \frac{1}{n_1m} \sum_{k \geq 2} \frac{\theta_2^k}{n_0^{k-1}} \sum_{i_1, \dots, i_{2k}}^\ast \mathbf{E}  \partial_{Y^\ast_{i_2 i_3}} \cdots \partial_{Y^\ast_{i_{2k} i_1}}  \left (Y^\ast G \right )_{i_2 i_1},
\end{split}
\end{equation}
where the summations \(\sum^\ast\) are understood over pairwise distinct indices. Here in the second line the factorial \((k-1)!\) disappears since there are exactly \((k-1)!\) ways to map the variables \(Y_{i_3 i_4},Y_{i_5 i_6} \dots, Y_{i_{2k-1} i_{2k}}\) into \(Y^\ast_{i_2 i_3}, Y_{i_3 i_4}, \dots, Y^\ast_{i_{2k} i_1}\) with distinct \(i_1, \dots, i_{2k}\). From this point onwards, we will omit reference to \(\mathbf{E}\) to simplify notation slightly. 

We now need to compute the partial derivatives in~\eqref{eq:3.7a}. The proof of the following lemma is included in Appendix~\ref{appendixA}.
\begin{lem}\label{lem3.6}
Let \(G(z) = (M-z)^{-1}\), \(z \in \mathbb{H}\), be the resolvent of the random matrix \(M = \frac{1}{m}YY^\ast \in \mathbb{R}^{n_1 \times n_1}\). Then, it holds that
\begin{subequations}
  \begin{align}
    \partial_{Y^\ast_{i_2 i_1}} \left ( Y^\ast G \right )_{i_2 i_1} &= G_{i_1 i_1} \left ( 1 -  \left ( \frac{ Y^\ast G Y }{m} \right)_{i_2 i_2} \right ),\\
    \partial_{Y^\ast_{i_2 i_3}} \cdots \partial_{Y^\ast_{i_{2k} i_1}}  \left (Y^\ast G \right )_{i_2 i_1} &\approx -\partial_{Y_{i_3 i_4}} \cdots \partial_{Y_{i_{2k-1} i_{2k}}} \left (\frac{GY}{m} \right )_{i_3 i_{2k}}G_{i_1 i_1} \left ( 1-  \left ( \frac{Y^\ast G Y}{m} \right )_{i_2 i_2} \right ).
  \end{align}
\end{subequations}
\end{lem} 
Thus, using Lemma~\ref{lem3.6} in~\eqref{eq:3.7a} we have 
\begin{equation}\label{1+zg first eq}
\begin{split}
1 + z g &\approx \frac{\theta_1}{n_1m} \sum_{i_1, i_2}^\ast G_{i_1 i_1} \left ( 1 -  \left ( \frac{ Y^\ast G Y }{m} \right)_{i_2 i_2} \right )\\
& \quad - \frac{1}{n_1 m} \sum_{k \geq 2} \frac{\theta_2^k}{n_0^{k-1}} \sum_{i_1, \dots, i_{2k}}^\ast \partial_{Y_{i_3 i_4}} \cdots \partial_{Y_{i_{2k-1} i_{2k}}} \left (\frac{GY}{m} \right )_{i_3 i_{2k}}G_{i_1 i_1} \left ( 1-  \left ( \frac{Y^\ast G Y}{m} \right )_{i_2 i_2} \right )\\
& = \theta_1 g - \theta_1 \frac{n_1}{m} g \left \langle \frac{Y^\ast G Y}{m} \right \rangle \\
& \quad - \left (  g - \frac{n_1}{m} g \left \langle \frac{Y^\ast G Y}{m} \right \rangle  \right ) \, \frac{1}{m}\sum_{k \geq 2} \frac{\theta_2^k}{n_0^{k-1}} \sum_{i_3, \dots, i_{2k}}^\ast  \partial_{Y_{i_3 i_4}} \cdots \partial_{Y_{i_{2k-1} i_{2k}}} \left (GY \right )_{i_3 i_{2k}},
\end{split}
\end{equation}
where \(\left \langle \frac{Y^\ast G Y}{m} \right \rangle \coloneqq \frac{1}{n_1} \Tr \frac{Y^\ast G Y}{m} = 1 + zg\) from~\eqref{eq:3.2}. Again, we stress that the equalities are meant in expectation. Moreover, shifting the index in the above summation, we get
\begin{equation*}
\begin{split}
& \frac{1}{m} \sum_{k \geq 2} \frac{\theta_2^k}{n_0^{k-1}} \sum_{i_3, \dots, i_{2k}}^\ast \partial_{Y_{i_3 i_4}} \cdots \partial_{Y_{i_{2k-1} i_{2k}}} \left (GY \right )_{i_3 i_{2k}} \\
&=\theta_2 \frac{n_1}{n_0} \frac{1}{m} \sum_{k \geq 1} \frac{\theta_2^{k}}{n_1 n_0^{k-1}} \sum_{i_3, \dots, i_{2k+2}}^\ast \partial_{Y_{i_3 i_4}} \cdots \partial_{Y_{i_{2k+1} i_{2k+2}}} \left (GY \right )_{i_3 i_{2k+2}}\\
&= \theta^2_2 \frac{n_1}{n_0} \, \frac{1}{n_1 m}\sum_{i_3, i_4}^\ast  \partial_{Y_{i_3 i_4}} \left (GY \right )_{i_3 i_4} \\
& \quad + \theta_2 \frac{n_1}{n_0} \frac{1}{n_1 m} \sum_{k \geq 2} \frac{\theta_2^{k}}{n_0^{k-1}} \sum_{i_3, \dots, i_{2k+2}}^\ast  \partial_{Y_{i_3 i_4}} \cdots \partial_{Y_{i_{2k+1} i_{2k+2}}} \left (GY \right )_{i_3 i_{2k+2}}\\
& \approx \theta^2_2 \frac{n_1}{n_0} \left ( g - \frac{n_1}{m}  g \left \langle \frac{Y^\ast G Y}{m} \right \rangle \right ) + \theta_2 \frac{n_1}{n_0} \left ( 1+z  g - \theta_1 g + \theta_1 \frac{n_1}{m} g \left \langle \frac{Y^\ast G Y}{m} \right \rangle \right )\\
&= \theta_2 \frac{n_1}{n_0} (1+z g) - \theta_2 (\theta_1-\theta_2) \frac{n_1}{n_0} g \left ( 1 - \frac{n_1}{m} (1+zg) \right ),
\end{split}
\end{equation*}
where in the third step we used~\eqref{eq:3.7a}. Finally, together with~\eqref{1+zg first eq}, we have 
\begin{equation}\label{final approx a}
\begin{split}
1+z g &\approx \theta_1 g \left ( 1 - \frac{n_1}{m} (1 + z g) \right ) - \theta_2 \frac{n_1}{n_0} g (1+z g) \left ( 1 - \frac{n_1}{m}(1+zg) \right ) \\
& \quad + \theta_2 (\theta_1-\theta_2) \frac{n_1}{n_0} g^2 \left ( 1 - \frac{n_1}{m}(1+zg) \right )^2,
\end{split}
\end{equation}
which corresponds to the desired equation~\eqref{self-const eq} as \(n_0, n_1, m \to \infty\). Thus,~\eqref{final approx a} combined with the concentration inequality given in Lemma~\ref{lem3.7} completes the proof of Theorem~\ref{thm2.2}. 
\begin{proof}[Proof of Theorem~\ref{thm2.2}]
We need to show the concentration w.r.t.\ \(\mathbf{E}_{W,X} \equiv \mathbf{E}\). By the triangle and Jensen inequality we have 
\[  
\begin{split}
  \mathbf E \lvert g(z)-\mathbf E g(z)\rvert^4 &\lesssim \mathbf E \lvert g(z)-\mathbf E_W g(z)\rvert^4 + \mathbf E_X \lvert \mathbf E_W g(z)-\mathbf E g(z)\rvert^4 \\
  &\le \mathbf E_X \Bigl(\mathbf E_W \lvert g(z)-\mathbf E_W g(z)\rvert^4\Bigr) + \mathbf E_W\Bigl( \mathbf E_X \lvert g(z)-\mathbf E_X g(z)\rvert^4\Bigr) \lesssim \frac{2}{n_1^2 (\Im z)^4} 
\end{split}\] 
and thus the almost sure convergence follows from the Borel-Cantelli Lemma, completing the proof of Theorem~\ref{thm2.2} together with~\eqref{final approx a}. 
\end{proof}
\subsection{Proof of Proposition~\ref{prop:3.5}} \label{subsection3.4}
In light of the central limit theorem, we have that in the asymptotic limit the random variables
\[\left ( \frac{WX}{\sqrt{n_0}} \right )_{ij} = \frac{1}{\sqrt{n_0}} \sum_{k=1}^{n_0} W_{ik} X_{kj},\]
are approximately \(\mathcal{N}(0,\sigma_w^2\sigma_x^2)\)-normally distributed. Our next goal is to compute their cumulants. The first cumulant or expectation vanishes identically. For the second cumulant we obtain: 
\begin{lem}\label{lem A.2}
The cumulant of \(\frac{ (WX)_{i_1 i_2}}{\sqrt{n_0}}\) and \(\frac{ (WX)_{i_3 i_4}}{\sqrt{n_0}}\) is nonzero only if \(i_1=i_3\) and \(i_2=i_4\), and in this case it holds that
\[\kappa \left (\frac{ (WX)_{i_1 i_2}}{\sqrt{n_0}}, \frac{ (WX)^\ast_{i_2 i_1}}{\sqrt{n_0}} \right ) = \sigma^2_w \sigma^2_x.\]
\end{lem}
\begin{proof}
We have 
\begin{equation*}
\begin{split}
\kappa \left (\frac{ (WX)_{i_1 i_2}}{\sqrt{n_0}}, \frac{ (WX)_{i_3 i_4}}{\sqrt{n_0}} \right ) & = \frac{1}{n_0} \mathbf{E} (WX)_{i_1 i_2} (WX)_{i_3 i_4} \\
&= \frac{1}{n_0} \sum_{k_1, k_2=1}^{n_0} \mathbf{E} W_{i_1 k_1} X_{k_1 i_2} W_{i_3 k_2} X_{k_2 i_4}\\
 &= \frac{1}{n_0} \sum_{k_1=1}^{n_0} \delta_{i_1 i_3} \delta_{i_2  i_4} \: \mathbf{E} W^2_{i_1 k_1} X^2_{k_1 i_2} = \delta_{i_1 i_3} \delta_{i_2  i_4} \sigma_w^2 \sigma_x^2.
\end{split}
\end{equation*}
Thus, the second cumulant is nonzero if \(i_1=i_3\) and \(i_2=i_4\), and in this case it is exactly the variance of the random variable \(\frac{ (WX)_{ij}}{\sqrt{n_0}}\). 
\end{proof}

We now consider four random entries, and we compute 
\[\frac{1}{n_0^2} \kappa \Bigl((WX)_{i_1 i_2},  (WX)_{i_3 i_4}, (WX)_{i_5 i_6}, (WX)_{i_7 i_8}\Bigr).\]
We observe that the cumulant vanishes identically if any index appears exactly once by the independence property, and thus each \(W\)- and \(X\)-index must appear exactly twice. This is only possible if we have two cycles on two indices each, or a single four-cycle. The cumulant of the former vanishes identically by independence ant thus the only non-vanishing 4-cumulant is 
\begin{equation*}
\begin{split}
& \kappa \left (\frac{ (WX)_{i_1 i_2}}{\sqrt{n_0}},\frac{ (WX)^\ast_{i_2 i_3}}{\sqrt{n_0}}, \frac{ (WX)_{i_3 i_4}}{\sqrt{n_0}}, \frac{ (WX)^\ast_{i_4 i_1}}{\sqrt{n_0}}  \right) \\
& = \frac{1}{n_0^2} \mathbf{E} (WX)_{i_1 i_2} (WX)^\ast_{i_2 i_3} (WX)_{i_3 i_4}  (WX)^\ast_{i_4 i_1} \\
& = \frac{1}{n_0^2}\sum_{k_1, k_2, k_3, k_4=1}^{n_0} \mathbf{E} W_{i_1 k_1} X_{k_1 i_2} W_{i_3 k_2} X_{k_2 i_2} W_{i_3 k_3} X_{k_3 i_4} W_{i_1 k_4} X_{k_4 i_4}  \\
&=  \frac{1}{n^2_0}\sum_{k_1=1}^{n_0} \mathbf{E} W^2_{i_1 k_1}X^2_{k_1 i_2} W^2_{i_3 k_1} X^2_{k_1 i_4} = \frac{ \left ( \sigma_w^2 \sigma_x^2 \right )^2}{n_0} 
\end{split}
\end{equation*}
Here for the first equality we used~\eqref{cum comb} where all but the trivial partition vanish identically since in some expectation a single index appears. %
This result can be generalised:

\begin{lem}\label{lemma3.3}
For \(k\ge 2\) and pairwise distinct indices we have
\begin{equation*}
\begin{split}
& \kappa \left (\frac{ (WX)_{i_1 i_2}}{\sqrt{n_0}},\frac{ (WX)^\ast_{i_2 i_3}}{\sqrt{n_0}}, \frac{ (WX)_{i_3 i_4}}{\sqrt{n_0}}, \dots, \frac{ (WX)^\ast_{i_{2k} i_1}}{\sqrt{n_0}}  \right) = \frac{ \left ( \sigma_w^2 \sigma_x^2 \right )^k}{n_0^{k-1}} + \mathcal{O}(n_0^{-k}).
\end{split}
\end{equation*}
\end{lem}
\begin{proof}
As illustrated for the case with four random variables, to have a nonzero cumulant, we can encode the \(2k\) random variables as a cycle graph of length \(2k\). Then, the only contribution comes from
\begin{equation*}
\kappa \left (\frac{ (WX)_{i_1 i_2}}{\sqrt{n_0}}, \dots, \frac{ (WX)^\ast_{i_{2k} i_1}}{\sqrt{n_0}}  \right) = \frac{1}{n_0^k} \, \mathbf{E} (WX)_{i_1 i_2} \cdots (WX)^\ast_{i_{2k} i_1} = \frac{ \left ( \sigma_w^2 \sigma_x^2 \right )^k}{n_0^{k-1}} + \mathcal{O}(n_0^{-k}),
 \end{equation*}
which completes the proof.
\end{proof}
Finally, we compute the cumulants of the entries of the random matrix \(Y\). Since the activation function \(f\) is applied component-wise, it follows from the previous results that the only contribution comes from \(\kappa( Y_{i_1 i_2}, Y^\ast_{i_2 i_3}, Y_{i_3 i_4}, \dots, Y^\ast_{i_{2k} i_1} )\) for \(k \geq 1\) and \(i_1, \dots, i_{2k}\) distinct, thus proving that \(Y\) has cycle correlations. 
\begin{proof}[Proof of Proposition~\ref{prop:3.5}]
From the Berry-Ess\'een Theorem it follows that 
\[
\begin{split}
  \kappa(Y_{ij}) &= \mathbf{E} Y_{ij} = \int_\mathbb{R} f(x) \frac{e^{-x^2/2 \sigma_w^2 \sigma_x^2}}{\sigma_w \sigma_x \sqrt{2 \pi}} \dif x+\mathcal{O}(n_0^{-1/2})\\
  &= \int_\mathbb{R} f(\sigma_w \sigma_x x) \frac{e^{-x^2/2}}{\sqrt{2 \pi}} \dif x+\mathcal{O}(n_0^{-1/2}) = \mathcal{O}(n_0^{-1/2}),
\end{split}\] 
and
\[\kappa(Y_{ij}, Y^\ast_{ji}) =  (1+\mathcal{O}(n_0^{-1/2}))\int_\mathbb{R} f^2(\sigma_w \sigma_x x) \frac{e^{-x^2/2}}{\sqrt{2 \pi}} \text{d}x=\theta_1(f)(1+\mathcal{O}(n_0^{-1/2})),\]
since the random variables \((WX)_{ij}/\sqrt{n_0}\) are approximately centred Gaussian with variance \(\sigma_w^2\sigma_x^2\). Let \(k > 1\). Then, since \(f\) is a smooth function with compact support, we have that \(f\) is in \(C^l\) for some integer \(l>1 + \frac{2k^2}{k-1}\). Using the Fourier inversion theorem, it follows that
\begin{equation*}
\begin{split}
f(x_1) & = \frac{1}{2\pi} \int_\mathbb{R} \hat{f}(t_1) \, e^{it_1x_1} \text{d}t_1\\
& = \frac{1}{2\pi} \int_{|t_1| \leq n_0^{\frac{k-1}{2k}}} \hat{f}(t_1) \, e^{it_1x_1} \text{d}t_1 + \frac{1}{2\pi} \int_{|t_1| > n_0^{\frac{k-1}{2k}}} \hat{f}(t_1) \, e^{it_1x_1} \text{d}t_1\\
& = \frac{1}{2\pi} \int_{|t_1| \leq n_0^{\frac{k-1}{2k}}} \hat{f}(t_1) \, e^{it_1x_1} \text{d}t_1 + \mathcal{O}\left ( (n_0^{\frac{k-1}{2k}})^{1-l} \right),
\end{split}
\end{equation*}
where we used \(|\hat{f}(t_1)| \leq \frac{c}{(1 + |t_1|)^l}\), for some positive constant \(c\). For notational simplicity we work in the case \(k=2\), but the argument when \(k >2\) is the same. We compute
\begin{equation*}
  \begin{split}
  &\kappa(Y_{i_1 i_2}, Y^\ast_{i_2 i_3}, Y_{i_3 i_4}, Y^\ast_{i_4 i_1}) \\
  & = \frac{1}{(2 \pi)^4} \int_{\forall i, \, |t_i| \leq n_0^\frac{1}{4}} \hat{f}(t_1) \hat{f}(t_2) \hat{f}(t_3) \hat{f}(t_4) \kappa(e^{i t_1 Z_{i_1 i_2}},e^{i t_2 Z^\ast_{i_2 i_3}},e^{i t_3Z_{i_3 i_4}},e^{i t_4 Z^\ast_{i_4 i_1}}) \dif\bm{t} + \mathcal{O}(n_0^{-2}),\\
  &= \frac{1}{(2 \pi)^4} \sum_{l_1,\ldots,l_4\ge 1}\int_{\forall i, \, |t_i| \leq n_0^\frac{1}{4}} \prod_{i=1}^4\Bigl(\hat{f}(t_i) \frac{(i t_i)^{l_i}}{l_i!}\Bigr) \kappa((Z_{i_1 i_2})^{l_1},(Z^\ast_{i_2 i_3})^{l_2},(Z_{i_3 i_4})^{l_3}, (Z^\ast_{i_4 i_1})^{l_4}) \dif\bm{t} + \mathcal{O}(n_0^{-2})
  \end{split}
\end{equation*}
where we introduced \(Z:=WX/\sqrt{n_0}\) and in the second equality used that any cumulant involving the deterministic \(1\) vanishes identically. We now expand the cumulant involving powers of \(Z\) via the well known formula~\cite[Theorem 11.30]{nica_speicher_2006} in terms of partitions of the set \(\{1,\ldots,l_1+l_2+l_3+l_4\}\) whose joint with the partition \(\{\{1,\ldots,l_1\},\ldots,\{l_1+l_2+l_3+1,\ldots,+l_1+l_2+l_3+l_4\}\}\) is the trivial partition. By the independence property it is clear that the leading contribution comes from those partitions with one block connecting one copy of each of \(Z_{i_1 i_2},Z^\ast_{i_2 i_3},Z_{i_3 i_4},Z^\ast_{i_4 i_1}\) and the remaining blocks being internal pairings. Since for odd \(l_i\) there are \(l_1!!\cdots l_4!!\) such partitions it follows that 
\begin{equation*}
  \begin{split}
  &\kappa(Y_{i_1 i_2}, Y^\ast_{i_2 i_3}, Y_{i_3 i_4}, Y^\ast_{i_4 i_1}) \\
  &\quad= \frac{1}{(2 \pi)^4} \sum_{\substack{l_1,\ldots,l_4\ge 1\\ l_i\text{ odd}}}\int_{\forall i, \, |t_i| \leq n_0^\frac{1}{4}} \prod_{i=1}^4\Bigl(\hat{f}(t_i) \frac{(i t_i)^{l_i}}{(l_i-1)!!}\Bigr) \kappa(Z_{i_1 i_2},Z^\ast_{i_2 i_3},Z_{i_3 i_4},Z^\ast_{i_4 i_1})\\
  &\quad\qquad\qquad\qquad\qquad\times \Var(Z_{i_1 i_2})^{(l_1-1)/2}\cdots \Var(Z^\ast_{i_4 i_1})^{(l_4-1)/2}  \dif\bm{t} + \mathcal{O}(n_0^{-3/2})\\
  &\quad = \frac{\sigma_w^4\sigma_x^4}{n_0}\frac{1}{(2 \pi)^4} \sum_{\substack{k_1,\ldots,k_4\ge 0}}\int_{\forall i, \, |t_i| \leq n_0^\frac{1}{4}} t_1t_2t_3t_4 \prod_{i=1}^4\Bigl(\hat{f}(t_i) \frac{(-\sigma_w^2\sigma_x^2 t_i^2/2)^{k_i}}{k_i!}\Bigr) \dif\bm{t} + \mathcal{O}(n_0^{-3/2})\\
  &\quad = \frac{1}{n_0} \left ( \sigma_w \sigma_x \frac{1}{2 \pi} \int \widehat{f'}(t) e^{- \sigma_w^2 \sigma_x^2 \, t^2/2} \, \text{d}t \right )^4 + \mathcal{O}(n_0^{-3/2}),
  \end{split}
\end{equation*}
where in the penultimate step we used Lemmata~\ref{lem A.2}--\ref{lemma3.3} and in the ultimate step we used the Fourier property \(\widehat{f'}(t) = i t \hat{f}(t)\). 
Together with 
\begin{equation*}
\begin{split}
\frac{\sigma_w \sigma_x}{2 \pi}\int \widehat{f'}(t) e^{- \sigma_w^2 \sigma_x^2 \, t^2/2} \, \text{d}t &= \frac{1}{\sqrt{2 \pi}}\int f'(x) e^{- x^2/ 2\sigma_w^2\sigma_x^2} \, \text{d}x \\
&= \sigma_w \sigma_x \int f'(\sigma_w \sigma_x x) \frac{e^{- x^2/2}}{\sqrt{2 \pi}} \, \text{d}x = \theta_2(f)^{1/2}.
\end{split}
\end{equation*}
we conclude
\begin{equation*}
\begin{split}
\kappa(Y_{i_1 i_2} ,Y^\ast_{i_2 i_3} ,Y_{i_3 i_4} ,Y^\ast_{i_4 i_1}) &= \theta_2(f)^2  n_0^{-1}\Bigl(1 + \mathcal{O}(n_0^{-1/2})\Bigr),
\end{split}
\end{equation*}
just as claimed. 
\end{proof}

\section{Proof of Theorem~\ref{thm2.6}} 
\subsection{Derivation of the self-consistent equation}

We proceed as in Subsection~\ref{subsection3.3}. We know from~\eqref{eq:3.2} that 
\begin{equation} \label{eq:4.3}
\frac{1}{m} \sum_{i=1}^{m} \left (\frac{Y^\ast G Y}{m} \right )_{ii} = \frac{n_1}{m} \left \langle \frac{Y Y^\ast G}{m} \right \rangle = \frac{n_1}{m} (1+zg).
\end{equation}
We further claim the following.
\begin{lem}\label{lem4.3}
It holds that
\begin{equation}\label{eq:lem4.3}
\frac{1}{m} \sum_{i=1}^{m} \sum_{j=1}^{n_1} \left (\frac{Y^\ast G Y}{m} \right )_{ij} = 1 + \mathcal{O}\left ( ( \theta_{1,b}(f) \, n_1)^{-1} \right ).
\end{equation}
\end{lem}
Together with~\eqref{eq:4.3}, Lemma~\ref{lem4.3} implies
\begin{equation} \label{eq:4.4}
\frac{1}{m} \sum_{i\neq j} \left (\frac{Y^\ast G Y}{m} \right )_{ij} \approx 1 - \frac{n_1}{m} (1+zg).
\end{equation}
\begin{proof}
Using the Woodbury matrix identity\footnote{For \(A \in \mathbb{R}^{n \times n}\), \(C \in \mathbb{R}^{r \times r}\), \(U \in \mathbb{R}^{n \times r}\) and \(V \in \mathbb{R}^{r \times n}\) the \emph{Woodbury matrix identity} is given by \[\left (A + UCV \right)^{-1} = A^{-1} - A^{-1} U \left (C^{-1} + V A^{-1} U \right)^{-1} V A^{-1}.\]}, we have
\[\frac{1}{m} \left ( \frac{Y^\ast G Y}{m} \right ) = \frac{1}{m^2} Y^\ast \left (\frac{Y Y^\ast}{m} - z  \right )^{-1} Y = \frac{1}{m} + \frac{z}{m}  \left (\frac{Y^\ast Y}{m} - z \right )^{-1},\]
which implies
\[ \sum_{i,j} \frac{1}{m} \left ( \frac{Y^\ast G Y}{m} \right )_{ij} = \sum_{i,j} \frac{1}{m} \delta_{ij} + \sum_{i,j} \frac{z}{m} \left (\frac{Y^\ast Y}{m} - z \right )_{ij}^{-1} = 1 + \sum_{i,j} \frac{z}{m} \left (\frac{Y^\ast Y}{m} - z \right )_{ij}^{-1}.\] 
So, we need to show that \(\sum_{i,j} \frac{z}{m} \left (\frac{Y^\ast Y}{m} - z \right )_{ij}^{-1}\) is approximately zero. Let \(e \coloneqq \frac{1}{\sqrt{m}} [1 \, \cdots \,1]^T \) be a normalized vector in \(\mathbb{R}^{m}\). We then write
\begin{equation*}
\begin{split}
\sum_{i,j} \frac{z}{m} \left (\frac{Y^\ast Y}{m} - z \right )_{ij}^{-1} = z \, \langle e,  \left (\frac{Y^\ast Y}{m} - z \right )^{-1} e \rangle.
\end{split}
\end{equation*}
It turns out that \(e\) is approximately an eigenvector of \(\frac{1}{m}Y^\ast Y\). Indeed, it holds that
\[\mathbf{E} \left ( \frac{Y^\ast Y}{m} e \right)_i  = \frac{1}{m\sqrt{m}} \sum_{j=1}^m \sum_{k=1}^{n_1} \mathbf{E} \, Y^\ast_{ik} Y_{kj} \approx m^{-1/2} n_1 \, \theta_{1,b}(f) = (n_1 \, \theta_{1,b}(f)) e_i.\]
Moreover, the variance is approximately \(\mathcal{O}(n_1/m)\), which means that the standard deviation is of order \(1\), while the expectation of order \(n_1\). Thus, \(e\) is approximately an eigenvector of \(\frac{1}{m}Y^\ast Y\) with eigenvalue \(n_1 \theta_{1,b}(f)\). Since \(\theta_{1,b}(f)\) is nonzero by assumption, we have that \(e\) is approximately an eigenvector of the matrix \(\left (\frac{Y^\ast Y}{m} - z \mathbf{1}_{m}\right )^{-1}\) with eigenvalue \((n_1 \theta_{1,b}(f) -z)^{-1}\), from which the result follows:
\[ 
\left | \langle e,  \left (\frac{Y^\ast Y}{m} - z \right )^{-1} e \rangle \right | \approx \left | (n_1 \, \theta_{1,b}(f) -z)^{-1} \right | \ll 1.\qedhere \]
\end{proof}

Given Lemma~\ref{lem4.3} and Proposition~\ref{prop:4.1}, we can now prove the global law for the random matrix \(M\) with the cycle correlations.
\begin{proof}[Proof of Theorem~\ref{thm2.6}]
Applying Proposition~\ref{prop:4.1} to~\eqref{eq:3.3} and using the same power counting argument as in~\eqref{eq:3.7a} we obtain
\begin{equation}\label{eq:4.5}
\begin{split}
1+z g &\approx \frac{1}{n_1 m} \sum_{i_1, i_2}^\ast \kappa(Y_{i_1 i_2}, Y^\ast_{i_2 i_1}) \, \partial_{Y^\ast_{i_2 i_1}} \left (Y^\ast G \right )_{i_2 i_1} + \frac{1}{n_1 m} \sum_{i_1, i_2, i_3}^\ast \kappa(Y_{i_1 i_2}, Y^\ast_{i_3 i_1}) \, \partial_{Y^\ast_{i_3 i_1}} \left (Y^\ast G \right )_{i_2 i_1}  \\
& \quad + \frac{1}{n_1m} \sum_{k \geq 2} \sum_{i_1, \dots, i_{2k}}^\ast  \kappa(Y_{i_1 i_2}, \dots, Y^\ast_{i_{2k} i_1}) \, \partial_{Y^\ast_{i_2 i_3}} \cdots \partial_{Y^\ast_{i_{2k} i_1}}  \left ( Y^\ast G \right )_{i_2 i_1}\\
& \approx \frac{\theta_1(f)}{n_1m} \sum_{i_1, i_2}^\ast \partial_{Y^\ast_{i_2 i_1}} \left (Y^\ast G \right )_{i_2 i_1} + \frac{\theta_{1,b}(f)}{n_1m} \sum_{i_1} \sum_{i_2 , i_3}^\ast  \partial_{Y^\ast_{i_3 i_1}} \left (Y^\ast G \right )_{i_2 i_1}  \\
& \quad + \frac{1}{n_1m} \sum_{k \geq 2} \frac{\theta_2^k(f)}{n_0^{k-1}} \sum_{i_1, \dots, i_{2k}}^\ast  \partial_{Y^\ast_{i_2 i_3}} \cdots \partial_{Y^\ast_{i_{2k} i_1}}  \left (Y^\ast G \right)_{i_2 i_1},
\end{split}
\end{equation}
where we omitted reference to \(\mathbf{E}\) to simplify notation. Given Lemma~\ref{lem3.6}, we only need to compute \(\partial_{Y^\ast_{i_3 i_1}} \left (Y^\ast G\right )_{i_2 i_1}\):
\[ \partial_{Y^\ast_{i_3 i_1}} \left (Y^\ast G \right )_{i_2 i_1} = \sum_{j=1}^{n_1}  \partial_{Y^\ast_{i_3 i_1}} \left (Y^\ast_{i_2 j} G_{j i_1}\right ) \approx - G_{i_1 i_1} \left ( \frac{Y^\ast G Y}{m} \right )_{i_2 i_3},\] 
where we omitted the contribution of \( \partial_{Y^\ast_{i_3 i_1}} Y^\ast_{i_2 j} \) since it is very small. Plugging the partial derivatives into~\eqref{eq:4.5}, we get
\begin{equation*}
\begin{split}
1+z g & \approx \frac{\theta_1(f)}{n_1 m} \sum_{i_1, i_2}^\ast  G_{i_1 i_1} \left (1 - \left ( \frac{Y^\ast G Y}{m} \right)_{i_2 i_2} \right ) - \frac{\theta_{1,b}(f)}{n_1 m} \sum_{i_1} \sum_{i_2,i_3}^\ast G_{i_1 i_1}\left ( \frac{Y^\ast G Y}{m} \right)_{i_2 i_3}  \\
& \quad - \frac{1}{n_1m} \sum_{k \geq 2} \frac{\theta_2^k(f)}{n_0^{k-1}} \sum_{i_1, \dots, i_{2k}}^\ast \partial_{Y_{i_3 i_4}} \cdots \partial_{Y_{i_{2k-1} i_{2k}}}  \left (\frac{GY}{m} \right )_{i_3 i_{2k}} G_{i_1 i_1} \left (1 - \left ( \frac{Y^\ast G Y}{m} \right)_{i_2 i_2} \right )\\
& \approx \theta_1(f) g \left ( 1 - \frac{n_1}{m} (1+z g) \right ) - \theta_{1,b}(f) g \left ( 1 - \frac{n_1}{m} (1+zg) \right ) \\
& \quad - g \left ( 1 - \frac{n_1}{m} (1+zg) \right ) \sum_{k \geq 2} \frac{\theta_2^k}{n_0^{k-1}} \sum_{i_3, \dots, i_{2k}}^\ast \partial_{Y_{i_3 i_4}} \cdots \partial_{Y_{i_{2k-1} i_{2k}}}  \left (\frac{GY}{m} \right )_{i_3 i_{2k}},
\end{split}
\end{equation*}
where in the second step we used~\eqref{eq:4.3} and~\eqref{eq:4.4}. Finally, by shifting the index in the summation and doing some simple bookkeeping, we have 
\begin{equation*}
\begin{split}
1+zg &\approx ( \theta_1 - \theta_{1,b})g \left ( 1 - \frac{n_1}{m} (1+zg) \right ) - \theta_{2} \frac{n_1}{n_0} g (1+zg) \left ( 1 - \frac{n_1}{m} (1+zg) \right ) \\
& \quad + \theta_2 (\theta_1 - \theta_{1,b}-\theta_2) \frac{n_1}{n_0}g^2 \left ( 1 - \frac{n_1}{m} (1+zg) \right )^2, 
\end{split}
\end{equation*}
which corresponds to the self-consistent equation~\eqref{self-const eq} as \(n_0, n_1, m \to \infty\), where \(\theta_1\) is replaced by \(\theta_1 - \theta_{1,b}\). In the same way as in the bias-free case, the concentration inequality of Lemma~\ref{lem3.7} can also be applied here, thereby concluding that \(g\) is approximately equal to its mean with high probability. The first claim of Theorem~\ref{thm2.6} then follows. The second claim follows easily from Lemma~\ref{lem4.3}. Since \(n_1\theta_{1,b}(f)\) is approximately an eigenvalue of the random matrix \(\frac{1}{m}Y^\ast Y\), and  since the nonzero eigenvalues of \(Y^\ast Y\) are the same as the one of \(YY^\ast\), we have that \(\lambda_{\max} \approx n_1 \theta_{1,b}(f)\) is an eigenvalue of \(M\) located away from the rest of the spectrum (called \emph{outlier}). This concludes the proof of Theorem~\ref{thm2.6}.
\end{proof}

\subsection{Proof of Proposition~\ref{prop:4.1}}\label{subsection4.2}
In light of the central limit theorem, in the asymptotic limit the random variables \(\frac{(WX)_{ij}}{\sqrt{n_0}} + B_i\) are approximately normally distributed with zero mean and variance \(\sigma_w^2 \sigma_x^2 + \sigma_b^2\). In contrast to the bias-free case, here we have two different nonzero second cumulants of the entries of the random matrix \(\frac{WX}{\sqrt{n_0}} + B\), and therefore also of the \(Y_{ij}\)'s.

\begin{proof}[Proof of Proposition~\ref{prop:4.1}]
The first identity follows in a straightforward manner by assumption~\eqref{eq:2.7}:
\[\kappa(Y_{ij}) = \mathbf{E} Y_{ij} = \int_\mathbb{R} f(x) \frac{ e^{-x^2/2(\sigma_w^2 \sigma_x^2 + \sigma_b^2)} }{\sqrt{2 \pi (\sigma_w^2 \sigma_x^2 + \sigma_b^2)} } \, \text{d}x+\mathcal{O}(n_0^{-1/2}) = \mathcal{O}(n_0^{-1/2}).\]
For the second cumulant, we first compute
\begin{equation*}
\begin{split}
\kappa \left (\frac{(WX)_{i_1 i_2}}{\sqrt{n_0}} + B_{i_1}, \frac{(WX)_{i_3 i_4}}{\sqrt{n_0}} + B_{i_3} \right ) &= \mathbf{E} \left (\frac{(WX)_{i_1 i_2}}{\sqrt{n_0}} + B_{i_1} \right )  \left (\frac{(WX)_{i_3 i_4}}{\sqrt{n_0}} + B_{i_3} \right ) \\
& =  \frac{1}{n_0} \mathbf{E} (WX)_{i_1 i_2} (WX)_{i_3 i_4}  +  \mathbf{E} B_{i_1} B_{i_3} \\
& = \delta_{i_1 i_3} \delta_{i_2 i_4} \, \sigma_w^2 \sigma_x^2 + \delta_{i_1 i_3} \sigma_b^2.
\end{split}
\end{equation*}
For \(i_1=i_3\) and \(i_2=i_4\), the cumulant \(\kappa(Y_{i_1 i_2}, Y_{i_2 i_1}^\ast)\) follows easily:
\begin{equation*}
\kappa(Y_{i_1 i_2}, Y^\ast_{i_2 i_1} ) = (1+\mathcal{O}(n_0^{-1/2}))\int_{\mathbb{R}} f^2(x) \frac{e^{- x^2 / 2(\sigma_w^2 \sigma_x^2 + \sigma_b^2)}}{\sqrt{2 \pi (\sigma_w^2 \sigma_x^2 + \sigma_b^2)} } \, \text{d}x = \theta_1(f)(1+\mathcal{O}(n_0^{-1/2})).
\end{equation*}
On the other hand, for \(i_1=i_3\) and \(i_2 \neq i_4\), to compute the cumulant \(\kappa(Y_{i_1 i_2}, Y_{i_4 i_1}^\ast)\), we need the characteristic function of \(\frac{(WX)_{i_1 i_2}}{\sqrt{n_0}}+B_{i_1}\) and \(\frac{(WX)^\ast_{i_4 i_1}}{\sqrt{n_0}}+B_{i_1}\)
which turns out to be asymptotically equal to 
\begin{equation*}
\exp \left ( - \frac{\sigma_w^2 \sigma_x^2+\sigma_b^2}{2} (t_1^2 + t_2^2) - \sigma_b^2 t_1 t_2 \right ).
\end{equation*}
Now, we can compute the cumulant of \(Y_{i_1 i_2}\) and \(Y^\ast_{i_4 i_1}\):
\begin{equation*}
\begin{split}
\kappa(Y_{i_1 i_2}, Y^\ast_{i_4 i_1}) & \approx \frac{1}{(2 \pi)^2} \int_{\mathbb{R}^2} f(x_1) f(x_2) e^{-i \bm{t} \cdot \bm{x}} \exp \left ( - \frac{\sigma_w^2 \sigma_x^2+\sigma_b^2}{2} (t_1^2 + t_2^2) - \sigma_b^2 t_1 t_2 \right ) \text{d}\bm{t} \,\text{d}\bm{x}\\
& =  \frac{1}{(2 \pi)^2} \int_{\mathbb{R}^2} \hat{f}(t_1)\hat{f}(t_2) \exp \left ( - \frac{\sigma_w^2 \sigma_x^2+\sigma_b^2}{2} (t_1^2 + t_2^2) - \sigma_b^2 t_1 t_2 \right ) \text{d}t_1 \, \text{d}t_2,
\end{split}
\end{equation*}
where in the second step we applied the Fourier inversion theorem. We denote the covariance matrix \(\Sigma\) by 
\begin{equation} \label{eq:4.7}
\Sigma \coloneqq \begin{pmatrix}
\sigma_w^2 \sigma_x^2+\sigma_b^2 & \sigma_b^2 \\
 \sigma_b^2 & \sigma_w^2 \sigma_x^2+\sigma_b^2  
\end{pmatrix}
\end{equation}
with determinant \(\det (\Sigma) = \sigma_w^2 \sigma_x^2 ( \sigma_w^2 \sigma_x^2+2\sigma_b^2 )\) and inverse matrix
\[\Sigma^{-1} = \frac{1}{\det (\Sigma) } \begin{pmatrix}
\sigma_w^2 \sigma_x^2+\sigma_b^2 & -\sigma_b^2 \\
-\sigma_b^2 & \sigma_w^2 \sigma_x^2+\sigma_b^2  
\end{pmatrix}.\]
Again applying the Fourier inversion formula, we obtain
\begin{equation*}
\begin{split}
\kappa(Y_{i_1 i_2}, Y^\ast_{i_4 i_1}) &\approx  \frac{1}{(2 \pi)^2} \int_{\mathbb{R}^2} \hat{f}(t_1)\hat{f}(t_2) e^{- \frac{1}{2} \langle \bm{t}, \Sigma \bm{t} \rangle} \text{d}\bm{t} \\
&= \frac{1}{(2 \pi)^2} \int_{\mathbb{R}^2} f(x_1)f(x_2) \frac{2 \pi}{\sqrt{\det(\Sigma)}} e^{- \frac{1}{2} \langle \bm{x}, \Sigma^{-1} \bm{x} \rangle} \text{d}\bm{x}\\
& = \frac{1}{2\pi \sqrt{  \sigma_w^2 \sigma_x^2 ( \sigma_w^2 \sigma_x^2+2\sigma_b^2 )}} \int_{\mathbb{R}^2} f(x_1) f(x_2) e^{- \frac{1}{2} \langle \bm{x}, \Sigma^{-1} \bm{x} \rangle} \text{d}\bm{x}=\theta_{1,b}(f),
\end{split}
\end{equation*}
where 
\[e^{- \frac{1}{2} \langle \bm{x}, \Sigma^{-1} \bm{x} \rangle} = \exp \left (- \frac{(\sigma_w^2 \sigma_x^2+\sigma_b^2)(x_1^2+x_2^2)-2\sigma_b^2x_1x_2}{2\sigma_w^2 \sigma_x^2 ( \sigma_w^2 \sigma_x^2+2\sigma_b^2 ) } \right).\] 

To complete the proof, it remains to compute the joint cumulant of \(Y_{i_1 i_2}, Y^\ast_{i_2 i_3}, Y_{i_3 i_4}, \dots,Y^\ast_{i_{2k} i_1}\) for \(k >1\) and \(i_1, \dots, i_{2k}\) distinct. For notational simplicity, we prove the statement for \(k=2\). First, we use the cumulant asymptotics in order to asymptotically compute the characteristic function. The cumulants have match those of the bias-free case, except for 
\[\kappa \left ( \frac{(WX)_{i_1 i_2}}{\sqrt{n_0}} + B_{i_1},  \frac{(WX)_{i_1 i_2}}{\sqrt{n_0}} + B_{i_1} \right ) = \sigma_w^2 \sigma_x^2 + \sigma_b^2 .\] 
In addition to all these cumulants, we also have  
\[\kappa \left ( \frac{(WX)_{i_1 i_2}}{\sqrt{n_0}} + B_{i_1} ,  \frac{(WX)^\ast_{i_4 i_1}}{\sqrt{n_0}} + B_{i_1}  \right ) = \kappa \left ( \frac{(WX)^\ast_{i_2 i_3}}{\sqrt{n_0}} + B_{i_3} ,  \frac{(WX)_{i_3 i_4}}{\sqrt{n_0}} + B_{i_3}  \right )= \sigma_b^2.\]
Therefore, the \(\log\)-characteristic function is given by
\begin{equation*}
\begin{split}
&-\frac{\sigma_w^2 \sigma_x^2+\sigma_b^2}{2} \sum_{i=1}^4 t_i^2 - \sigma_b^2 (t_1 t_4 + t_2 t_3) + \sum_{n\geq 1} \frac{(-1)^{n-1}}{n} \left ( \frac{( \sigma_w^2 \sigma_x^2)^2}{n_0} \prod_{i=1}^4 t_i +\mathcal{O}(n_0^{-2}) \right )^n\\
&= -\frac{\sigma_w^2 \sigma_x^2+\sigma_b^2}{2} \sum_{i=1}^4 t_i^2 - \sigma_b^2 (t_1 t_4 + t_2 t_3) + \log \left ( 1+\frac{( \sigma_w^2 \sigma_x^2)^2}{n_0} \prod_{i=1}^4 t_i  +\mathcal{O}(n_0^{-2}) \right ),
\end{split}
\end{equation*}
for \(t_1, t_2, t_3, t_4\in \mathbb{R}\) such that \(\abs{t_i}<n_0^{1/4}\). We obtain the characteristic function by taking the exponential of the above expression. By the same argument as in the proof of Proposition~\ref{prop:3.5}, we have
\begin{equation*}
\begin{split}
&\kappa(Y_{i_1 i_2},Y^\ast_{i_2 i_3},Y_{i_3 i_4},Y^\ast_{i_4 i_1})\\
& = %
 \frac{1}{n_0} \left ( \frac{\sigma_w^2 \sigma_x^2}{(2\pi)^2} \int \widehat{f'}(t_1)\widehat{f'}(t_2) \exp \left ( - \frac{\sigma_w^2\sigma_x^2+\sigma_b^2}{2} (t_1^2 + t_2^2) - \sigma_b^2 t_1t_2 \right ) \text{d}t_1 \text{d}t_2  \right)^2 +\mathcal{O}(n_0^{-3/2})\\
& = \left (\frac{1}{2\pi \sqrt{\sigma_w^2\sigma_x^2(\sigma_w^2\sigma_x^2+2\sigma_b^2)}} \int f(x_1) f(x_2) e^{- \frac{1}{2} \langle \bm{x}, \Sigma^{-1}\bm{x} \rangle} \text{d} \bm{x} \right )^2\\
&\quad + \frac{1}{n_0} \left (\frac{\sigma_w^2 \sigma_x^2}{2\pi \sqrt{\sigma_w^2\sigma_x^2(\sigma_w^2\sigma_x^2+2\sigma_b^2)}} \int f'(x_1) f'(x_2) e^{- \frac{1}{2} \langle \bm{x}, \Sigma^{-1}\bm{x} \rangle} \text{d} \bm{x} \right )^2 + \mathcal{O}(n_0^{-3/2}),
\end{split}
\end{equation*}
where \(\Sigma\) is the matrix defined by~\eqref{eq:4.7}. It then follows that
\begin{equation*}
\begin{split}
\kappa(Y_{i_1 i_2}, Y^\ast_{i_2 i_3},Y_{i_3 i_4},Y^\ast_{i_4 i_1})& \approx \mathbf{E} Y_{i_1 i_2} Y^\ast_{i_2 i_3} Y_{i_3 i_4} Y^\ast_{i_4 i_1} -  \mathbf{E} Y_{i_1 i_2} Y^\ast_{i_4 i_1} \: \mathbf{E}\, Y^\ast_{i_2 i_3} Y_{i_3 i_4} \\
& = \theta_2(f)^2  n_0^{-1}\Bigl(1 + \mathcal{O}(n_0^{-1/2})\Bigr),
\end{split}
\end{equation*}
as desired. The proof for \(k>2\) is similar.
\end{proof}

\section{Proofs of auxiliary results}\label{appendixA}
\begin{proof}[Proof of Lemma~\ref{lemma3.1}]
By applying the Fourier inversion theorem, we have
\[\mathbf{E} X_1 f(\bm{X}) = \frac{1}{(2\pi)^n} \int_{\mathbb{R}^n} \int_{\mathbb{R}^n} x_1 f(\bm{x}) e^{-i \bm{t} \cdot \bm{x}} \varphi_{\bm{X}}(\bm{t}) \text{d} \bm{x} \, \text{d}\bm{t},\]
where \(\varphi_{\bm{X}}(\bm{t})\) is the characteristic function of the \(n\)-dimensional random vector \(\bm{X}\). It holds that \(\int_{\mathbb{R}^n} (-i x_1) f(\bm{x}) e^{-i \bm{t} \cdot \bm{x}} \text{d} \bm{x} = \partial_{t_1} \hat{f}(\bm{t})\). Then, it follows that
\begin{equation*}
\begin{split}
\mathbf{E}  X_1 f(\bm{X}) & = \frac{i}{(2\pi)^n} \int_{\mathbb{R}^n} \left ( \partial_{t_1} \hat{f}(\bm{t}) \right ) \varphi_{\bm{X}}(\bm{t})\text{d}\bm{t} \\
& = - \frac{i}{(2\pi)^n} \int_{\mathbb{R}^n}  \hat{f}(\bm{t}) \Big ( \partial_{t_1}\varphi_{\bm{X}}(\bm{t}) \Big) \text{d}\bm{t}\\
& =  - \frac{i}{(2\pi)^n} \int_{\mathbb{R}^n}  \hat{f}(\bm{t}) \Big ( \partial_{t_1} e^{\log \varphi_{\bm{X}}(\bm{t})} \Big) \text{d}\bm{t}\\
& = - \frac{i}{(2\pi)^n} \int_{\mathbb{R}^n}  \hat{f}(\bm{t}) \Big ( \partial_{t_1} \log \varphi_{\bm{X}}(\bm{t}) \Big) \varphi_{\bm{X}}(\bm{t}) \text{d}\bm{t}. 
\end{split}
\end{equation*}
Cumulants can also be defined in an analytical way as the coefficients of the \(\log\)-characteristic function
\begin{equation}\label{cum anal}
\log \mathbf{E} e^{i \bm{t} \cdot \bm{X}} = \sum_{\bm{l}} \kappa_{\bm{l}} \frac{ (i \bm{t})^{\bm{l}} }{\bm{l} !},
\end{equation}
where \(\sum_{\bm{l}}\) is the sum over all multi-indices \(\bm{l}=(l_1, \dots, l_n) \in \mathbb{N}^n\). We note that \(\kappa_{\bm{l}} (X_1, \dots, X_n) = \kappa ( \{ X_1 \}^{l_1}, \dots, \{ X_n \}^{l_n} )\) means that \(X_i\) appears \(l_i\) times. One can prove that this definition of cumulants is equivalent to  the combinatorial one given by~\ref{cum comb} (see~\cite{MR725217} for a proof). Using definition~\eqref{cum anal} results in 
\[ \partial_{t_1} \log \varphi_{\bm{X}}(\bm{t}) = i \sum_{\bm{l}} \kappa_{\bm{l}+\bm{e_1}} \frac{(i \bm{t})^{\bm{l}}}{\bm{l}!} ,\]
where \(\bm{l} + \bm{e_1} = (l_1+1, l_2, \dots, l_n)\). Since \((i \bm{t})^{\bm{l}} \hat{f}(\bm{t}) = \widehat{f^{(\bm{l})}}(\bm{t})\), we finally obtain
\begin{equation*}
\mathbf{E} X_1 f(\bm{X}) = \sum_{\bm{l}} \frac{\kappa_{\bm{l}+\bm{e_1}}}{\bm{l}!} \frac{1}{(2\pi)^n}\int_{\mathbb{R}^n} \widehat{f^{(\bm{l})}}(\bm{t})  \varphi_{\bm{X}}(\bm{t}) \text{d}\bm{t} = \sum_{\bm{l}} \frac{\kappa_{\bm{l}+\bm{e_1}}}{\bm{l}!} \mathbf{E} \, f^{(\bm{l})} (\bm{X}),
\end{equation*}
where we again applied the Fourier inversion formula. 
\end{proof}

\begin{proof}[Proof of Lemma~\ref{lem3.6}]
Let \(\Delta^{i,j}\) denote a \(m \times n_1\) matrix such that \(\Delta^{i,j}_{k l}=\mathbf{1}_{ \{(i,j) = (k,l)\} }\). Then, applying the resolvent identity, we get
\begin{equation*}
\frac{\partial G}{\partial Y^\ast_{i j}} = \lim_{\epsilon \to 0} \frac{\left ( \frac{Y ( Y^\ast + \epsilon \Delta^{i,j} )}{m}- z\right )^{-1}-\left ( \frac{Y Y^\ast}{m} - z\right )^{-1} }{\epsilon} = - \frac{G Y \Delta^{i,j} G}{m}.
\end{equation*}
It follows that \(\partial_{Y^\ast_{ij}} G_{ab} = - \left ( \frac{GY}{m} \right )_{a i} G_{jb}\) for \(1 \leq a,b \leq n_1\), \(1 \leq i \leq m\), and \(1 \leq j \leq n_1\). Therefore, we have
\begin{equation*}
\begin{split}
\partial_{Y^\ast_{i_2 i_1}} \left ( Y^\ast G \right )_{i_2 i_1} = \sum_{j=1}^{n_1}  \partial_{Y^\ast_{i_2 i_1}} \left (Y^\ast_{i_2 j} G_{j i_1} \right ) = G_{i_1 i_1} \left ( 1 -  \left ( \frac{ Y^\ast G Y }{m} \right)_{i_2 i_2} \right ),
\end{split}
\end{equation*}
which proves (3.6a). We now compute 
\begin{equation*}
\begin{split}
\sum_{j=1}^{n_1} \partial_{Y^\ast_{i_{2} i_3}} \partial_{Y^\ast_{i_{2k} i_1}}  \left ( Y^\ast_{i_2 j}  G_{j i_1} \right ) & \approx - \sum_{j=1}^{n_1} \partial_{Y^\ast_{i_{2} i_3}} \left ( Y^\ast_{i_2 j} \left ( \frac{GY}{m} \right)_{j i_{2k}}  G_{i_1 i_1}  \right )\\
& \approx - \left ( \frac{GY}{m} \right )_{i_3 i_{2k}} G_{i_1 i_1}  +  \left (\frac{Y^\ast GY}{m}\right)_{i_2 i_2} \left ( \frac{GY}{m} \right )_{i_3 i_{2k}} G_{i_1 i_1}, 
\end{split}
\end{equation*}
where the approximation in the first line comes from the fact that the contribution of \( \partial_{Y^\ast_{i_{2k} i_1}}  Y^\ast_{i_2 j}\) is very small and can therefore be neglected. Since the off-diagonals of the resolvent of random matrices are small if \(\Im z \gg n_1^{-1}\), the partial derivative \(\partial_{Y^\ast_{i_2 i_3}} G_{i_1 i_1}\) can be omitted. This justifies the second approximation. So, we obtain
\begin{equation*}
\begin{split}
\partial_{Y^\ast_{i_2 i_3}} \cdots \partial_{Y^\ast_{i_{2k} i_1}}  \left ( Y^\ast G \right )_{i_2 i_1} \approx -\partial_{Y_{i_3 i_4}} \cdots \partial_{Y_{i_{2k-1} i_{2k}}} \left (\frac{GY}{m} \right )_{i_3 i_{2k}}G_{i_1 i_1} \left ( 1-  \left ( \frac{Y^\ast G Y}{m} \right )_{i_2 i_2} \right ),
\end{split}
\end{equation*}
which completes the proof of Lemma~\ref{lem3.6}.
\end{proof}

\section{Concentration inequality}\label{appendixC}
\begin{proof}[Proof of Lemma~\ref{lem3.7}]
Without loss of generality, it suffices to prove the statement w.r.t.\ \(\mathbf{E}_X\) since by cyclicity the statement for \(\mathbf{E}_W\) is  analogous. We write \(X= (\bm{x}_1, \dots, \bm{x}_m)\) with \(\bm{x}_k = (x_{1k}, \dots, x_{n_0 k})'\), and similarly, \(Y = (\bm{y}_1, \dots, \bm{y}_m)\). We denote by \(\mathcal{F}_k\), \(1 \leq k \leq m\), the filtration generated by \(\{\bm{x}_l, \, 1 \leq l \leq k\}\) and by \(\mathbf{E}_k[\cdot] \coloneqq \mathbf{E}_X[\cdot \, | \, \mathcal{F}_k]\) the conditional expectation w.r.t.\ \(\mathcal{F}_k\). Now, we decompose \(g(z) - \mathbf{E}_X g(z)\) as a sum of martingale differences 
\[D_k \coloneqq \mathbf{E}_k \Tr(M-z\mathbf{1}_{n_1})^{-1} - \mathbf{E}_{k-1} \Tr(M-z\mathbf{1}_{n_1})^{-1} , \quad \text{for} \enspace k=1, \dots, m.\]
By construction, we have \(\mathbf{E}_m \Tr(M-z\mathbf{1}_{n_1})^{-1}= \Tr(M-z\mathbf{1}_{n_1})^{-1}\) and \(\mathbf{E}_0 \Tr(M-z\mathbf{1}_{n_1})^{-1} = \mathbf{E}_X \Tr(M-z\mathbf{1}_{n_1})^{-1}\). It then follows that
\[g(z) - \mathbf{E}_X  g(z) = \frac{1}{n_1} \sum_{k=1}^m  \mathbf{E}_k \Tr(M-z\mathbf{1}_{n_1})^{-1} - \mathbf{E}_{k-1} \Tr(M-z\mathbf{1}_{n_1})^{-1}= \frac{1}{n_1} \sum_{k=1}^m D_k.\]
Next, we define \(M_k \coloneqq M - \bm{y}_k \bm{y}_k^\ast\). We note that
\[\mathbf{E}_k \Tr(M_k-z\mathbf{1}_{n_1})^{-1}= \mathbf{E}_{k-1} \Tr(M_k-z\mathbf{1}_{n_1})^{-1},\]
since \(M_k\) is independent of \(\bm{y}_k\) and therefore is also independent of \(\bm{x}_k\). So, we have
\[D_k = (\mathbf{E}_k - \mathbf{E}_{k-1})[\Tr(M-z\mathbf{1}_{n_1})^{-1} - \Tr(M_k - z\mathbf{1}_{n_1})^{-1}].\]
Then, by the Shermann-Morrison formula, we have
\begin{equation*}
\begin{split}
\left | \Tr(M-z\mathbf{1}_{n_1})^{-1} - \Tr(M_k-z\mathbf{1}_{n_1})^{-1} \right |& = \left | \frac{\bm{y}_k^\ast (M_k - z\mathbf{1}_{n_1})^{-2} \bm{y_k}}{1 + \bm{y}_k^\ast (M_k - z\mathbf{1}_{n_1})^{-1} \bm{y_k}} \right |\\
& \leq \frac{|\bm{y}_k^\ast (M_k - z\mathbf{1}_{n_1})^{-2} \bm{y_k}|}{\Im(\bm{y}_k^\ast (M_k - z\mathbf{1}_{n_1})^{-1} \bm{y_k})} \\
& \leq \frac{1}{\Im z},
\end{split}
\end{equation*}
where the last inequality follows from the resolvent identity:
\begin{equation*}
\begin{split}
|\bm{y}_k^\ast (M_k - z\mathbf{1}_{n_1})^{-2} \bm{y_k}| & \leq \bm{y}_k^\ast (M_k - z\mathbf{1}_{n_1})^{-1} (M_k - \bar{z} \mathbf{1}_{n_1})^{-1} \bm{y_k} \\
& = \frac{\bm{y}_k^\ast  \left ((M_k - z\mathbf{1}_{n_1})^{-1} - (M_k - \bar{z} \mathbf{1}_{n_1})^{-1} \right ) \bm{y}_k}{2 i \, \Im z}\\
& = \frac{\Im(\bm{y}_k^\ast (M_k - z\mathbf{1}_{n_1})^{-1} \bm{y_k})}{\Im z}.
\end{split}
\end{equation*}
Thus, \(|D_k| \leq 2 (\Im z)^{-1}\), and so \(g(z) - \mathbf{E}_X g(z)\) is a sum of bounded martingale differences. We can now apply the Burkholder's inequality which states that for \(\{ D_k, 1 \leq k \leq m\}\) being a complex-valued martingale difference sequence, for \(p > 1\), 
\[\mathbf{E} \left | \sum_{k=1}^m D_k \right |^p \leq C \, \mathbf{E} \left ( \sum_{k=1}^n |D_k|^2 \right )^{p/2},\]
where \(C\) is a positive constant depending on \(p\). We refer to~\cite[Lemma 2.12]{MR2567175} for a proof of this inequality. By choosing \(p=4\), we get
\begin{equation*}
\begin{split}
\mathbf{E}_X \left | g(z) - \mathbf{E}_X g(z) \right |^4 &= \frac{1}{n_1^4} \, \mathbf{E}_X \left | \sum_{k=1}^m D_k \right |^4\\
& \leq \frac{1}{n_1^4}C \,  \mathbf{E}_X \left ( \sum_{k=1}^m |D_k|^2 \right )^2\\
& \leq \frac{16 \,C \, m^2}{n_1^4 \, (\Im z)^4} = \mathcal{O}(n_1^{-2} (\Im z)^{-4}),
\end{split}
\end{equation*}
just as claimed.
\end{proof}

\section{Complex case}
\begin{rmk}\label{remark complx}
  We can also consider matrices \(X \in \mathbb{C}^{n_0 \times m}\) and \(W \in \mathbb{C}^{n_1 \times n_0}\) of complex random entries with zero mean and variance \(\mathbf{E} |X_{ij}|^2=\sigma_{x}^2\) and \(\mathbf{E}|W_{ij}|^2=\sigma_{w}^2\). Let \(M = \frac{1}{m}Y Y^\ast\) with \(Y = f \left ( \frac{WX}{\sqrt{n_0}} \right)\), and let \(f\colon \mathbb{C} \to \mathbb{R}\) be a real-differentiable function satisfying 
  \(\int_{\mathbb{C}} f (\sigma_w \sigma_x z) \frac{e^{-|z|^2}}{\pi} \mathrm{d}^2 z=0\). Set \(\theta_1(f) = \int_\mathbb{C} |f(\sigma_w \sigma_x z)|^2 \, \frac{e^{-|z|^2}}{\pi} \mathrm{d}^2z\).
  Then, it can be proved that the normalized trace of the resolvent of \(M\) satisfies equation~\eqref{MP eq}.
\end{rmk}

\nocite{MR2760897}

\end{document}